\documentclass[twoside]{article}

\usepackage[preprint]{aistats2027}
\let\originalaistatsaddress\aistatsaddress
\renewcommand{\aistatsaddress}[1]{%
  \originalaistatsaddress{#1}%
  \removelastskip
  \vspace{1em}%
}

\usepackage{amsmath,amsfonts,bm}

\def\eqref#1{equation~\ref{#1}}

\def\1{\bm{1}}

\DeclareMathAlphabet{\mathsfit}{\encodingdefault}{\sfdefault}{m}{sl}
\SetMathAlphabet{\mathsfit}{bold}{\encodingdefault}{\sfdefault}{bx}{n}

\newcommand{\E}{\mathbb{E}}

\newcommand{\R}{\mathbb{R}}

\usepackage{hyperref}
\usepackage{url}
\usepackage{amsmath,amssymb,amsthm}
\usepackage{mathtools,mathrsfs}
\PassOptionsToPackage{dvipsnames}{xcolor}
\usepackage{tikz}
\usepackage{enumitem}
\usepackage{hyperref}
\usepackage{booktabs}
\usepackage{graphicx}          
\usepackage{subcaption}        
\graphicspath{{experiments_sizegen/}{../experiments_sizegen/}} 
\usepackage{colortbl}
\usepackage{multicol}
\usepackage[expansion=false]{microtype}
\usepackage[dvipsnames]{xcolor}
\usepackage{titlesec}
\usepackage{authblk}

\hypersetup{colorlinks=true, linkcolor=blue!60!black, urlcolor=blue!60!black, citecolor=blue!60!black}

\newcommand{\risk}{\mathcal R}
\newcommand{\err}{\mathcal E}
\newcommand{\calA}{\mathcal A}

\theoremstyle{plain}

\newtheorem{corollary}{Corollary}
\newtheorem{proposition}{Proposition}
\newtheorem{lemma}{Lemma}
\theoremstyle{definition}
\newtheorem{definition}{Definition}
\theoremstyle{remark}

\newtheorem{assumption}{Assumption}

\begin{document}

\onecolumn

\aistatstitle{Looped GNNs Beyond Fixed Depth: Graph Size Generalization and Parameter-Efficient Adaptation}
\aistatsauthor{Zhiyang Wang$^{1,*}$ \And Zhishang Luo$^{2}$ \And  Yusu Wang$^{2}$ }
\aistatsaddress{ $^{1}$Washington University in St.~Louis, St.~Louis, MO, USA \\  $^{2}$University of California, San Diego, La Jolla, CA, USA } 
{\renewcommand{\thefootnote}{*}\footnotetext{Preprint. This work was conducted while the author was at University of California, San Diego (UCSD).}}

\begin{abstract}
Generalizing from small graphs to larger ones can require computation beyond the depth of a neural model used during training. We study target tasks that admit a common local iterative computation and whose effective dependency radius grows with graph size.
The central idea is to learn a local update rule through GNN pretraining, then repeatedly apply the same learned processor beyond the training depth to extend information reach on larger graphs without retraining.
Calling such parameter-shared architecture a \emph{looped GNN}, we categorize six schemes for looped GNNs to separate the effects of depth extension, fine-tuning, and freeze-and-steer adaptation. 
Our analysis establishes fixed-depth information limits and sufficient conditions under which accurate iterative reuse overcomes them.
To improve beyond unadapted reuse, we introduce \emph{Loop-Freeze-Steer} (LFS), which freezes the pretrained processor and learns a small controller to steer its repeated application, using much fewer trainable parameters than full fine-tuning.
Our finite-sample analysis gives sufficient conditions for adaptation to improve performance on larger graphs.
Experiments on synthetic tasks and chip placement evaluate depth extension and adaptation, highlighting LFS's accuracy--trainable-parameter trade-off for size generalization.

\end{abstract}

\section{Introduction}\label{sec:intro}
Generalizing from small training graphs to larger deployment graphs is an important challenge in graph learning~\cite{yehudai2021from,ruiz2020graphon}. In many tasks, increasing graph size also increases the amount of local computation needed for an accurate prediction. In such settings, it is natural to adapt a model trained on smaller graphs so that it can operate effectively on larger ones. The overarching questions we are interested in understanding are what adaptation strategies are effective, which factors govern their success, and under what conditions adaptation can provide a genuine (ideally provably guaranteed) benefit.

In this paper, we take a step toward answering these questions by focusing on a family of graph targets that admit a common local iterative computation whose effective dependency radius grows with graph size. This radius measures how much neighborhood information is needed to meet a specified accuracy for the desired target function. A common update rule propagates information through repeated applications, motivating a learned processor that can be reused at greater deployment depths.

For tasks with a growing effective dependency radius, a fixed-depth GNN has insufficient information to meet the desired accuracy for larger graphs, even after parameter adaptation (Proposition~\ref{prop:ceiling}). This motivates extending computation at deployment. If a processor learned on small graphs captures the task's local update, repeatedly applying the same processor can propagate information farther on larger graphs without introducing new processor parameters. We refer to graph models with this repeated, parameter-shared computation as \emph{looped GNNs}. Our starting point is a GNN pretrained at a fixed depth, with unadapted looping reusing its processor beyond that depth without further parameter updates. 

Greater information reach and accurate computation are distinct requirements. Accurate prediction at the training depth does not by itself establish that the learned processor approximates the task's local update on states encountered during longer computations.
Additional iterations therefore create both an opportunity to access more information and a need to control accumulated update errors.
Our analysis makes this distinction explicit through a fixed-radius information floor and an alignment-based bound on iterative prediction error. Together, these results give sufficient conditions under which extended reuse improves on fixed-depth prediction (Lemma~\ref{lem:accum}, Proposition~\ref{prop:iter}).

Within an encode--process--decode formulation, we separately study deployment-depth extension and parameter adaptation.
Figure~\ref{fig:scheme-overview} organizes six schemes along these two axes: Fixed, Fixed-Tune, and Fixed-Steer retain the pretraining depth, while Loop, Loop-Tune, and Loop-Freeze-Steer allow adjustable deployment depth.
Within each family of the same depth, the model is used without adaptation, fine-tuned, or adapted through a controller around a frozen processor.
The adapted schemes use a limited labeled set of intermediate-size graphs before evaluation on larger graphs. Matched comparisons examine the effect of depth under each adaptation policy and the effect of adaptation at a given depth.

Accurate prediction at the training depth does not by itself ensure that the learned processor approximates the task's local update rule, which matters when it is reused for additional iterations. 
This motivates how adaptation can improve extended reuse while retaining useful pretrained updates.
Among the six schemes, our proposed \emph{Loop-Freeze-Steer} (LFS) addresses this question by freezing the pretrained processor and learning a small controller to steer the hidden states passed between iterations, combining depth extension with parameter-efficient adaptation. Its motivation is that a pretrained processor may provide a useful local update even when its repeated application requires adjustment on larger graphs. LFS freezes the encoder and processor and learns a small size-conditioned controller, together with the readout.
The controller acts \emph{between iterations}, modifying the hidden states entering subsequent processor applications. It thereby adapts the repeated computation itself, with the same controller parameters reused across steps and graph sizes. Compared with updating the full network, this design concentrates the trainable budget in a compact controller and readout. Its practical aim is to improve on unadapted reuse while approaching fine-tuning accuracy with far fewer trainable parameters.

Specifically, for adaptation, our finite-sample analysis distinguishes the correction achievable within an adaptation class, the cost of estimating that correction from limited data, and its persistence across graph sizes.
Adaptation improves expected deployment risk over its unadapted baseline when the achievable gain exceeds estimation cost and the signed loss of benefit across sizes
(Lemma~\ref{lem:P3}, Proposition~\ref{prop:adapt}).
This criterion applies to both fine-tuning and steering. It also separates the trainable-parameter savings of LFS from the statistical conditions needed for accurate adaptation.

We evaluate these questions on controlled synthetic tasks and real-world chip-placement problems. Synthetic experiments examine the effects of depth, update alignment, and adaptation, while chip-placement experiments test transfer to larger designs and additional distribution shifts.
The comparisons assess both predictive error and trainable parameter budget, reflecting the practical value of obtaining comparable accuracy with a much smaller adaptation model. Our contributions are summarized as follows:
\begin{itemize}
 \item \textbf{A framework for graph-size generalization.}
 We distinguish effective dependency radius from task-computation
 depth and organize six schemes to separate depth extension from adaptation. Our analysis establishes fixed-depth information limits and sufficient conditions for useful iterative reuse and finite-sample adaptation across sizes.

 \item \textbf{Loop-Freeze-Steer.}
 We introduce adaptation of looped computation through a shared
 size-conditioned controller between frozen processor applications, together with readout adaptation, reducing trainable parameters relative to processor fine-tuning when the controller is smaller.

 \item \textbf{Synthetic and real-world validation.}
 We evaluate depth extension and adaptation on controlled synthetic tasks and real-world chip-placement problems, assessing performance and parameter efficiency across graph sizes.
\end{itemize}

\section{Related works}\label{sec:related}
This work connects iterative neural architectures and algorithmic reasoning, graph size generalization, parameter-efficient adaptation, and learned-controller designs. We review each direction, then identify our contributions.

\subsection{Iterative neural networks and algorithmic alignment.} 
Iterative architectures repeatedly apply a shared operator, allowing computational depth to change without increasing processor parameter count. IterGNN~\cite{tang2020itergnn} is a form of looped GNN under our definition, while it combines a learned stopping mechanism for adaptive iteration with homogeneous transformations for numerical input scaling.
Algorithmic alignment~\cite{xu2020what} relates learnability to the match between network structure and the computation of a task. Neural algorithmic reasoning studies learning to execute algorithms~\cite{velickovic2020nar}. CLRS provides inputs, intermediate states, and outputs for 30 algorithms to evaluate generalization to larger instances~\cite{velickovic2022clrs}, while generalist learners train a shared processor across multiple algorithms \cite{velickovic2022generalist}.
Mahdavi et al.~\cite{mahdavi2022better} study out-of-distribution generalization in this setting, where larger instances can require longer reasoning trajectories.
For shortest paths, Nerem et al.~\cite{tang2025provable} establish conditions under which minimizing a sparsity-regularized training objective yields a GNN implementing Bellman--Ford, and hence generalizing to arbitrary instance sizes. 
Related work on looped Transformers studies shared-block computation for learning iterative algorithms~\cite{yang2024looped} and length generalization on tasks with iterative solutions~\cite{fan2024looped}.
Recurrent algorithmic models also examine reliable computation beyond training depths~\cite{bansal2022end}.
Deep equilibrium models~\cite{bai2019deq} and implicit GNNs~\cite{gu2020ignn} define representations through fixed points, while comparisons of implicit and unfolded GNNs examine the relation between equilibrium solutions and explicit iterative propagation~\cite{yang2022implicit}.
These approaches motivate reusable computation, but finite-horizon iteration and equilibrium prediction impose different requirements.
Our analysis separates computational horizon from effective dependency radius, and studies when a pretrained local update can be reused accurately at greater depth. In contrast to learning adaptive termination as in IterGNN, LFS adapts the hidden states passed between applications of a frozen pretrained processor, using a limited labeled adaptation set.

\subsection{Size generalization in graph learning.} 
An extensive theoretical literature studies when a \emph{trained} graph network's risk transfers across sizes. 
Graphon-based approaches compare finite-graph networks through a common graphon limiting object~\cite{ruiz2020graphon,ruiz2023transferability,maskey2023transferability,maskey2023graphon}.
Manifold-based work studies graph filters and neural networks through geometric limits and statistical generalization \cite{wang2023geometric,wang2025manifold, levie2021transferability}.
Graphops extend approximation and transferability analysis to sparse as well as dense graph sequences~\cite{le2023limits}; continuous-depth GNNs have also been studied through graphon differential equations and trajectory convergence~\cite{herbst2025continuous}.
For graph Transformers, manifold-limit analysis establishes size transferability through convolutional positional encodings~\cite{porras2026size}.

Local-structure analysis identifies changes in neighborhood distributions as a source of size-generalization failure and motivates learning interventions~\cite{yehudai2021from}.
Beyond transfer guarantees for a given predictor, DISGEN uses augmentations and disentangled representations to reduce reliance on size information~\cite{huang2024disgen}. Size-insensitive attention (SIA) addresses size-induced shifts in biological graphs using structural information motivated by spectral analysis \cite{buffelli2023sizeshift}.
Thus this literature includes both transferability theory and methods designed to improve size generalization. Our distinction is to analyze deployment-depth extension together with supervised adaptation of a pretrained model. The resulting bounds separate information reach, update accuracy, estimation from adaptation graphs, and transfer of the learned correction. Our analysis allows the pretrained model to change through depth extension, supervised parameter adaptation, or both. We study how these changes improve graph-size generalization, rather than only how a fixed predictor transfers across sizes.

\subsection{Parameter-efficient and test-time adaptation for graphs.} 
Adapters~\cite{houlsby2019adapters} and low-rank updates~\cite{hu2022lora} provide general strategies for adapting pretrained networks through a small set of parameters.
AdapterGNN~\cite{li2024adaptergnn} develops lightweight adapters for frozen GNN backbones and examines their generalization and empirical performance relative to full fine-tuning.
Graph prompting offers further adaptation mechanisms: GraphPrompt~\cite{liu2023graphprompt} unifies pretraining and downstream tasks through a shared task formulation;
GPPT~\cite{sun2022gppt} aligns downstream node classification with graph pretraining through prompt tuning; and GPF/GPF-plus~\cite{fang2024gpf} provide feature-based prompts applicable to different GNN pretraining strategies.
For equivariant graph models, ELoRA~\cite{wang2025elora} and MMEA~\cite{jin2025mmea} develop parameter-efficient updates that preserve the required symmetries.
Graph foundation-model approaches such as GFT and GIT~\cite{wang2024gft,zehong2025git} further study reusable representations across graphs and tasks. LFS specializes parameter-efficient adaptation to repeated computation on larger graphs: a shared controller acts between applications of a frozen processor, together with an adapted readout.
Its trainable parameter count is independent of the number of iterations. Fixed-Steer provides the corresponding fixed-depth control, and Loop-Tune permits full updates of the pretrained network.
Our analysis accounts separately for trainable parameters and loss-class complexity, and gives conditional adaptation guarantees rather than inferring accuracy superiority from parameter count alone.

Another adaptation setting uses unlabeled target data at test time: GraphTTA~\cite{chen2022graphtta} uses adversarial pseudo-group contrast, HomoTTT~\cite{zhang2024homottt} uses homophily-based contrastive learning, and Matcha~\cite{bao2024matcha} addresses graph structure shifts. This setting concerns when and from which data a model adapts, whereas parameter efficiency concerns which parameters are updated.
Our protocol uses limited labeled intermediate-size graphs for adaptation, then holds fitted parameters fixed on larger deployment graphs.

\subsection{Adaptive computation and iterative reuse.}
Adaptive Computation Time~\cite{graves2016act} and PonderNet~\cite{banino2021pondernet} learn input-dependent halting, while IterGNN~\cite{tang2020itergnn} incorporates adaptive termination into graph computation. Feature conditioning instead adjusts the computation at each step, as in FiLM's channel-wise affine modulation~\cite{perez2018film}. Training through long unrollings can involve vanishing or exploding gradients \cite{pascanu2013difficulty}, and repeated updates can accumulate errors beyond the training depth. Our approach separates fixed-depth pretraining from extended reuse: LFS freezes the pretrained processor and adapts a small controller and readout. The controller uses FiLM-style modulation conditioned on graph size and iteration to steer reuse, rather than learn when to halt. This reduces the trainable budget, while gradients still pass through the frozen processor; our alignment conditions address the accuracy of the complete controlled update.

\subsection{Positioning of our work.}
We organize six schemes to separate deployment-depth extension from parameter adaptation. The analysis establishes an information floor for fixed-depth models and sufficient alignment and task-resolution conditions for iterative reuse to improve on that floor.
LFS combines adjustable depth with parameter-efficient steering of a frozen pretrained processor. Its adaptation guarantee depends on achievable gain, estimation cost, and cross-size loss of benefit.
This provides a conditional account of useful adaptation and an explicit trainable-budget comparison with full fine-tuning; the relative accuracy of the two approaches depends on the task and adaptation distribution.

\section{Problem setup}\label{sec:fclass}
For each $N\in\mathbb N$, let $\mathcal G_N$ denote the space of attributed graphs with exactly $N$ nodes, and let $\mathcal P_N$ be a probability distribution supported on connected graphs in $\mathcal G_N$. We write $G_N\sim\mathcal P_N$ for a sampled graph. Write $B_k(G,v)$ for the rooted attributed $k$-hop neighborhood of $v$ in $G$, and $\operatorname{diam}(G)$ for the hop diameter.
Assume that there exists a deterministic scale $D_N\to\infty$ that describes the diameter growth of the graph familiy, i.e.,
\begin{equation}
\operatorname{diam}(G_N)=\Theta_{\mathbb P}(D_N), \qquad G_N\sim\mathcal P_N.
\end{equation}
We train on a finite labeled dataset of graphs of size $N_{\mathrm{train}}$ and evaluate on graphs drawn from $\mathcal P_{N_{\mathrm{OOD}}}$, where $N_{\mathrm{OOD}}>N_{\mathrm{train}}$.

Let $f_N^\star:\mathcal G_N\to\mathbb R^{N\times d_y}$ be a permutation-equivariant node-level target, and write $f_N^\star(G,v):=[f_N^\star(G)]_v \in\mathbb R^{d_y}$ as the target value conditionally on $G_N$, sampled $v\sim\operatorname{Unif}(V(G_N))$. Let $\mathscr B_{N,k}$ denote the sigma-algebra generated by $B_k(G_N, v)$. {Intuitively, $\mathscr B_{N,k}$ represents all information observable within this $k$-hop neighborhood of node $v$.} 
Assume $\mathbb E\|f_N^\star(G_N, v)\|_2^2<\infty$, we define the optimal radius-$k$ approximation error.
\begin{equation}\label{eqn:opterror}
   \mathcal{E}_N(k)
    :=
    \inf_{g\in L^2(\mathscr B_{N,k};\mathbb R^{d_y})}
    \E\!\left[
        \left\|
            f_N^\star(G_N, v)-g
        \right\|^2
    \right],
\end{equation}
where expectation is taken over the sampled graph and the uniformly sampled node.
For a fixed tolerance $\epsilon>0$, define the effective dependency radius as
\begin{equation}
    L_\epsilon(N)
:=
\min\{k\in\mathbb N_0:\mathcal{E}_N(k)\le\epsilon^2\}.
\end{equation}
{Comparing $L_\epsilon(N)$ with $D_N$ distinguishes targets that can be predicted to some accuracy from a bounded neighborhood. We define three regimes as follows.}
\begin{definition}[Dependency-radius regimes]
\label{def:classes}
Relative to the diameter scale $D_N$, we distinguish:
\begin{align*}
\mathcal F_1^{(\epsilon)}
\ \text{(bounded local)}:
&\qquad L_\epsilon(N)=O(1),\\
\mathcal F_2^{(\epsilon)}
\ \text{(growing sub-diameter)}:
&\qquad L_\epsilon(N)\to\infty, \\ &\qquad L_\epsilon(N)=o(D_N),\\
\mathcal F_3^{(\epsilon)}
\ \text{(diameter-scale)}:
&\qquad L_\epsilon(N)=\Theta(D_N).
\end{align*}
The classification depends on the target--distribution pair $(\{f_N^\star\}_N,\{\mathcal P_N\}_N)$ and on the tolerance $\epsilon$.
\end{definition}

\paragraph{Examples.}
Node degree, local clustering coefficients, and fixed-hop label counts are bounded-local. To illustrate genuinely growing reach, consider an $N$-node cycle graph with independent standard Gaussian scalar node attributes $X_u$ and diameter $D_N=\lfloor N/2\rfloor$.
Choose $k_N\to\infty$ with $k_N=o(N)$ and define
\begin{equation}\label{eq:example-local-sum}
    f_N^\star(G_N,v)
    =\frac{1}{\sqrt{2k_N+1}}
      \sum_{u\in B_{k_N}(G_N,v)}X_u.
\end{equation}
For sufficiently large $N$, $k_N<N/2$. When $k<k_N$, the optimal local predictor is the conditional expectation
\begin{equation}
     \mathbb E[f_N^\star(G_N,v)\mid\mathscr B_{N,k}]
 =\frac{1}{\sqrt{2k_N+1}}
   \sum_{u\in B_k(G_N,v)}X_u.
\end{equation}
It retains the target's normalization: conditioning reveals $2k+1$ attributes, while each unobserved attribute has conditional mean zero.
The residual contains $2k_N+1-(2k+1)$ independent unit-variance terms, so its mean squared error is their count divided by $2k_N+1$.
For $k\ge k_N$, the entire target is observed and the error is zero.
Together, these cases give
\begin{equation}\label{eq:example-local-error}
    \mathcal E_N(k)
    =\left(1-\frac{2k+1}{2k_N+1}\right)_+,
    \qquad (x)_+:=\max\{x,0\}.
\end{equation}
For fixed $0<\epsilon<1$, this gives $L_\epsilon(N)=\Theta(k_N)$.
Thus $k_N=\lceil\log N\rceil$ and $k_N=\lfloor\sqrt N\rfloor$ give two examples in $\mathcal F_2^{(\epsilon)}$, both with unit output variance. By contrast, the broadcast global sum $N^{-1/2}\sum_u X_u$ has $\mathcal E_N(k)=1-\min\{N,2k+1\}/N$, giving $L_\epsilon(N)=\Theta(N)$ and membership in $\mathcal F_3^{(\epsilon)}$.
Detailed derivations and the effects of normalization are given in Appendix~\ref{app:regime-examples}.

\subsection{Iterative operator form of the target}\label{sec:operator-form}
We consider {task} families with target functions 
generated by repeated applications of a local update rule.

\begin{definition}[Iterative operator form]\label{def:operator-form}
A size-relevant target family $\{f_N^\star\}_N$ admits an \emph{iterative operator form} if there exist size-independent encoder, update, and decoder templates $E^\star, U^\star, D^\star$, and horizons $T^\star(N)\in\mathbb N_0$, such that
\begin{equation}\label{eqn:operator-form}
    f_N^\star(G_N)
    =
    \left[
        D^\star \circ
\bigl(U^\star_{G_N}\bigr)^{T^\star(N)}
        \circ E^\star
    \right](X_{G_N}),
\end{equation}
where $X_{G_N}$ denotes the input node attributes and $E^\star$ as a node-wise encoder producing hidden states in dimension $d_h$,  $U^\star_{G_N}:\mathbb R^{N\times d_h}\to \mathbb R^{N\times d_h}$ as a common one-hop, permutation-equivariant update rule, $T^\star(N)$ as the task-side computational horizon, $D^\star$ as a node-wise decoder mapping the final hidden states to outputs in $\mathbb R^{N\times d_y}$.
We denote the class of task families admitting this representation by $\mathcal F^{\mathrm{op}}$. App.~\ref{app:iterative-tasks} gives examples satisfying this definition and explains the conditions needed for other local dynamics.
\end{definition}
The shared templates couple the targets across graph sizes and make operator reuse well-defined. Because each update propagates information by at most one hop, $L_\epsilon(N)\le T^\star(N)$. Our analysis focuses on $\mathcal F^{\mathrm{op}}\cap (\mathcal F_1^{(\epsilon)}\cup\mathcal F_2^{(\epsilon)})$. A concrete example is the damped local diffusion used in our synthetic experiments in Sec.~\ref{sec:exp-synthetic}.
App.~\ref{app:discussion} discusses representative tasks, their iterative structure, and the distinction between dependency radius and homogeneity.

\paragraph{Example: restart diffusion.}
On the attributed cycle above, let $P_{G_N}$ average the two neighbors of each node. Starting from $h_0^\star=0$, consider
\begin{equation}\label{eq:example-restart}
    h_{t+1}^\star
    =\alpha P_{G_N}h_t^\star+(1-\alpha)X_{G_N},
    \qquad 0<\alpha<1,
\end{equation}
with target $f_N^\star(G_N)=h_{T^\star(N)}^\star$ and identity decoder. The fixed source $X_{G_N}$ is locally available at each update. This update repeatedly combines neighboring states with a fixed node-local input, a restart mechanism used in label propagation and APPNP~\cite{zhou2003learning,klicpera2019appnp}.
Since $\|P_{G_N}\|_2\le1$, for $0\le k\le T^\star(N)$,
\begin{equation}\label{eq:example-restart-error}
    \mathbb E\!\left[
        \frac1N
        \|h_{T^\star(N)}^\star-h_k^\star\|_2^2
    \right]
    \le
    \alpha^{2k}
    \mathbb E\!\left[
        \frac1N\|X_{G_N}\|_2^2
    \right]
    =\alpha^{2k}.
\end{equation}
Thus fixed $\alpha$ and fixed tolerance give bounded sufficient depth, even when $T^\star(N)$ grows.
A growing horizon alone therefore does not establish a growing dependency radius.
The damped update used in our synthetic experiments in Sec. \ref{sec:exp-synthetic} has a related form. Its information requirements additionally depend on the input distribution, signal scale, and damping.
Appendix~\ref{app:iterative-tasks} discusses further iterative examples and their contraction conditions, while Appendix~\ref{app:homogeneity} distinguishes radius from magnitude scaling.

\subsection{Learned predictors and deployment risk}
\label{sec:learned-predictors}
All schemes use the encode--process--decode architecture with a parameter-shared processor applied iteratively for $K$ steps: 
\begin{equation}\label{eq:learned-predictor}
f_{N,K}(G_N)
= \left[
D_\eta\circ(U_{\theta,G_N})^K\circ E
\right](X_{G_N}),
\end{equation}
where $E$ and $D_\eta$ are node-wise encoder and decoder maps, and $U_{\theta,G_N}$ is the realization on $G_N$ of a parameter-shared processor. 
The task-side horizon $T^\star(N)$ defines the target computation, whereas $K$ is the chosen deployment depth of the learned model.
The expected per-node deployment risk is
\begin{align}
    \mathcal R_N( f)
:= \mathbb E
\left[
\frac{1}{N}\sum_{v\in V(G_N)}
\| f(G_N,v)-f_N^\star(G_N,v)\|_2^2
\right].
\end{align}\label{eq:deployment-risk}

\section{Iterative Reuse: From Fixed Depth to Loop-Freeze-Steer}\label{sec:schemes}
\begin{figure*}[t]
\centering
\resizebox{0.94\textwidth}{!}{%
\begin{tikzpicture}[
  x=\linewidth,y=0.9cm,
  every node/.style={
    font=\sffamily\small,
    align=center,inner sep=2pt},
  scheme/.style={
    draw=black!25,fill=black!3,rounded corners=2pt,
    text width=0.225\linewidth,
    minimum height=0.65cm,
    font=\sffamily\bfseries\small},
  block/.style={
    draw=black!40,fill=black!5,rounded corners=2pt,
    text width=0.18\linewidth,minimum height=0.65cm},
  trained/.style={block,fill=blue!7},
  arrow/.style={
    ->,>=stealth,draw=black!65,line width=0.6pt}
]
  \node[anchor=west,font=\sffamily\bfseries\small]
    at (0.01,1.42) {(a) Depth and adaptation};

  \node at (0.32,0.87) {No adaptation};
  \node at (0.58,0.87) {Fine-tune};
  \node at (0.84,0.87) {Freeze \& steer};

  \node[text width=0.16\linewidth] at (0.09,0)
    {Fixed depth\\$K_{\mathrm{fix}}$};
  \node[text width=0.16\linewidth] at (0.09,-0.98)
    {Adjustable depth\\$K$};

  \node[scheme] at (0.32,0) {Fixed};
  \node[scheme] at (0.58,0) {Fixed-Tune};
  \node[scheme] at (0.84,0) {Fixed-Steer};
  \node[scheme] at (0.32,-0.98) {Loop};
  \node[scheme] at (0.58,-0.98) {Loop-Tune};
  \node[scheme,fill=blue!7] at (0.84,-0.98)
    {Loop-Freeze-Steer (LFS)};

  \node[anchor=west,font=\sffamily\bfseries\small]
    at (0.01,-1.84) {(b) LFS computation};

  \node[block] (encoder) at (0.10,-3.2)
    {$E_0$\\frozen};
  \node[trained] (controller) at (0.365,-3.2)
    {$C_\phi$\\trainable};
  \node[block] (processor) at (0.63,-3.2)
    {$R_{\theta_0,G_N}$\\frozen};
  \node[trained] (decoder) at (0.895,-3.2)
    {$D_\eta$\\trainable};
  \node (ctx) at (0.365,-2.28)
    {$\mathrm{ctx}_N^{(k)}$};

  \draw[arrow] (encoder) --
    node[above] {$h_0$} (controller);
  \draw[arrow] (ctx) -- (controller);
  \draw[arrow] (controller) -- (processor);
  \draw[arrow] (processor) --
    node[above] {$h_K$} (decoder);
  \draw[arrow] (processor.south) -- (0.63,-4.08)
    -- node[below] {$K$ steps; shared weights} (0.365,-4.08)
    -- (controller.south);
\end{tikzpicture}%
}
\caption{(a) Six schemes by depth and adaptation.
(b) LFS steers each frozen processor application with a shared
controller; the readout acts after $K$ steps. Blue blocks are trainable.}
\label{fig:scheme-overview}
\end{figure*}
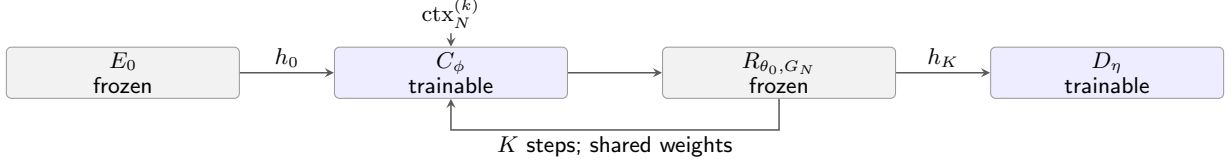

The target's iterative structure motivates reuse of a pretrained local processor on larger graphs. We organize six schemes along two axes: deployment depth (fixed or adjustable) and adaptation (none, fine-tuning, or freeze-and-steer), as summarized in Fig.~\ref{fig:scheme-overview}(a). 

\subsection{Shared backbone and deployment depth}\label{sec:shared-backbone}
Let $E_0$, $U_{\theta_0,G_N}$, and $D_{\eta_0}$ denote the node-wise encoder, processor, and node-wise decoder pretrained on graphs of size $N_{\mathrm{train}}$ at depth $K_{\mathrm{fix}}$.
Within each scheme, the same processor parameters are reused at every iteration. Fig.~\ref{fig:scheme-overview}(a) shows fixed deployment depth $K=K_{\mathrm{fix}}$ (top row) from adjustable depth $K=K(N)$ (bottom row). Without adaptation (leftmost column), both choices use the pretrained predictor
\begin{equation}\label{eq:shared-predictor}
 f_{N,K}(G_N)=
 \left[D_{\eta_0}\circ U_{\theta_0,G_N}^{\,K}\circ E_0\right](X_{G_N}),
\end{equation}
whose receptive radius is at most $K$.

All schemes are evaluated on larger graphs of size $N_{\mathrm{OOD}}>N_{\mathrm{train}}$. The adapted schemes update selected parameters using labeled graphs in $\mathcal D_{\mathrm{adapt}}$ of size $N_{\mathrm{adapt}}$, where $N_{\mathrm{train}}\le N_{\mathrm{adapt}}<N_{\mathrm{OOD}}$.

For the looped schemes, deployment depth $K(N)$ is a size-dependent design choice that allows computation to scale with the task's information requirements. Under additional assumptions on the task computation, App.~\ref{app:resolution-depth} and~\ref{app:geometric} derive sufficient depth bounds and a special case with logarithmic resolution depth.

\subsection{Deployment and adaptation schemes}\label{sec:comparison-schemes}
We compare three fixed-depth schemes with their adjustable-depth counterparts, culminating in our proposed Loop-Freeze-Steer (LFS).
Let $d_E,d_R,d_D,d_\phi$ count encoder, processor, readout, and controller parameters, respectively. Shared parameters are counted once, regardless of depth.
\paragraph{Fixed.}\label{scheme:fixed}
Deploy Eq. \ref{eq:shared-predictor} with $K=K_{\mathrm{fix}}$ and no adaptation. This is the pretrained model used as is.
\paragraph{Fixed-Tune.}\label{scheme:ftfix}
Fine-tune encoder, processor, and readout on $\mathcal D_{\mathrm{adapt}}$, retaining depth $K_{\mathrm{fix}}$. Writing $E_\xi$ for the adapted encoder, its predictor is $[D_\eta\circ R_{\theta,G_N}^{K_{\mathrm{fix}}}\circ E_\xi](X_{G_N})$, with $d_E+d_R+d_D$ trainable parameters.
\paragraph{Fixed-Steer.}\label{scheme:fsfix}
Keep $E_0$ and $R_{\theta_0,G_N}$ frozen and adapt the node-wise controller $C_\phi$ and readout $D_\eta$, using the LFS computation defined below at depth $K_{\mathrm{fix}}$.
This trains $d_\phi+d_D$ parameters.
All three fixed-depth schemes remain radius-$K_{\mathrm{fix}}$ local.
\paragraph{Loop.}\label{scheme:loop}
Run Eq.~\ref{eq:shared-predictor} at deployment-selected depth $K(N)$, without parameter updates. This isolates iterative reuse of the pretrained processor. Greater depth permits a larger receptive field, while useful depth extension depends on the task and the accuracy of repeated model updates.

\paragraph{Loop-Tune.}\label{scheme:ft}
Fine-tune encoder, processor, and readout on $\mathcal D_{\mathrm{adapt}}$, with deployment depth $K=K(N)$.
Like Fixed-Tune, this updates all $d_E+d_R+d_D$ pretrained parameters.
The distinction is adjustable deployment depth.

\paragraph{Loop-Freeze-Steer (LFS).}\label{scheme:fs}
Our proposed scheme combines depth extension with parameter-efficient adaptation by placing a shared, size-conditioned controller \emph{between repeated applications of a frozen processor}.
LFS freezes $E_0$ and $R_{\theta_0,G_N}$ and adapts only $C_\phi$ and $D_\eta$ as shown in Fig.~\ref{fig:scheme-overview}(b). Starting from $h_0=E_0(X_{G_N})$, for $0\le k<K(N)$ it computes
\begin{align}\label{eq:lfs-update}
 &\widetilde h_k=C_\phi(h_k,\mathrm{ctx}_N^{(k)}),\quad 
 h_{k+1}=U_{\theta_0,G_N}(\widetilde h_k),\\
&\qquad \qquad f_{N,K(N)}^{\mathrm{LFS}}(G_N)=D_\eta(h_{K(N)}).
\end{align}
The same controller parameters $\phi$ and frozen processor parameters $\theta_0$ are shared across iterations and graph sizes, giving a parameter budget independent of the number of iterations. The context $\mathrm{ctx}_N^{(k)}$ may encode graph size, iteration index, and node-local history, allowing steering to respond to the current computation with an encoding defined at all deployment depths.

The identity condition $C_0(h,\mathrm{ctx})=h$ embeds unadapted Loop within LFS: setting $\phi=0$ and $\eta=\eta_0$ recovers the pretrained computation at the same depth. This provides a baseline from which the controller and readout can learn task-specific corrections.

\section{Performance Analysis}\label{sec:analysis}

We analyze when aligned iteration overcomes fixed-radius information limits and when adaptation improves prediction across graph sizes.
Additional iterations expand information reach, while operator alignment controls the accuracy of the repeated computation. Fine-tuning and steering improve how the model uses this information when their adaptation gains persist across sizes and exceed estimation costs.
Throughout, $\mathsf S$ denotes one of the six schemes:
Fixed, Fixed-Tune, Fixed-Steer, Loop, Loop-Tune, or LFS.

\subsection{Analytical primitives}\label{sec:primitives}
With notations in Sec.~\ref{sec:fclass} and denote $Y_N=f_N^\star(G_N,v)$, we define
\[
 m_{N,k}^\star=\E[Y_N\mid\mathscr B_{N,k}],
 \qquad \err_N(k)=\E\|Y_N-m_{N,k}^\star\|^2.
\]
Conditional expectation gives the best squared-error prediction using only the rooted $k$-hop neighborhood. Thus $\err_N(k)$ is exactly the optimal local approximation error from the problem setup.

For the target's iterative representation, set $h_0^\star=E^\star(X_{G_N})$ and $h_{k+1}^\star=U_{G_N}(h_k^\star)$, and define
\[
 \widetilde Y_{N,k}=[D^\star(h_k^\star)]_v,\qquad
 a_N(k)=\E\|Y_N-\widetilde Y_{N,k}\|^2.
\]
Task encoders and decoders are node-wise and updates are one-hop local, so $\widetilde Y_{N,k}$ is $\mathscr B_{N,k}$-measurable.
We consider square-integrable task outputs and predictions; expectations include both graph and node sampling. For a finite-horizon target, $Y_N=\widetilde Y_{N,T_N^\star}$ and $a_N(T_N^\star)=0$.
At $k>T_N^\star$, $a_N(k)$ measures over-iteration error and need not decrease. An equilibrium target is the decoded limit of the task states.

By the conditional-expectation projection identity, every square-integrable radius-$k$ predictor $f$ satisfies
\begin{equation}\label{eq:projection}
 \risk_N(f)=\err_N(k)+\E\|f(G_N,v)-m_{N,k}^\star\|^2
 \ge\err_N(k).
\end{equation}
In particular, $\err_N(k)\le a_N(k)$, since stopping the task computation gives one radius-$k$ predictor.
A short justification is in App.~\ref{app:information}.

We denote  $\|H\|_N^2=N^{-1}\sum_{v\in V(G_N)}\|H_v\|^2$ as the norm for hidden-state and output arrays. In particular, $a_N(k)=\E_G\|D^\star(h_k^\star)-f_N^\star(G_N)\|_N^2$.

\begin{assumption}[Contractive target]\label{ass:contraction}
For each $G_N$, there exists $\rho_N\in(0,1)$ such that, 
\begin{equation}
    \|U_{G_N}(h)-U_{G_N}(\tilde h)\|_N
 \le\rho_N\|h-\tilde h\|_N
\end{equation}
for all hidden states in its state space.
\end{assumption}
Contractivity controls propagation of state errors under repeated task updates. It prescribes neither a resolution profile nor a depth law; both $\err_N(k)$ and $a_N(k)$ depend on the target--distribution pair. 

\begin{lemma}[Finite-horizon operator error]\label{lem:accum}
Under Assumption~\ref{ass:contraction}, fix a fitted scheme $\mathsf S$
and a depth $K$. Its encoder gives
\[
 h_0^{\mathsf S}=E_{\mathsf S}(X_{G_N}),
\]
and the same processor $U_{\theta,G_N}$ is reused for $0\le j<K$. Suppose, for $\mathcal P_N$-almost every $G_N$,
\begin{align*}
 \|h_0^{\mathsf S}-h_0^\star\|_N&\le\delta_{E,\mathsf S}(N),\\
 \|h_{j+1}^{\mathsf S}-U^\star_{G_N}(h_j^{\mathsf S})\|_N
 &\le\varepsilon_{\mathsf S}(N),\qquad 0\le j<K.
\end{align*}
Then
\begin{equation}\label{eq:state-bound}
 \|h_K^{\mathsf S}-h_K^\star\|_N
 \le\rho_N^K\delta_{E,\mathsf S}(N)
       +\frac{1-\rho_N^K}{1-\rho_N}\varepsilon_{\mathsf S}(N).
\end{equation}
If the learned decoder $D_\eta$ is $L_D$-Lipschitz in per-node norms and satisfies
$\|D_\eta(h_K^\star)-D^\star(h_K^\star)\|_N
 \le\delta_{D,\mathsf S}(N)$ almost surely,
then $f_{N,K}^{\mathsf S}(G_N)=D_\eta(h_K^{\mathsf S})$ satisfies
\begin{align}\label{eq:alignment-risk}
 \risk_N(f_{N,K}^{\mathsf S})
 \le{}&2\Bigg[L_D\left(\rho_N^K\delta_{E,\mathsf S}(N)
       +\frac{1-\rho_N^K}{1-\rho_N}\varepsilon_{\mathsf S}(N)\right) +\delta_{D,\mathsf S}(N)\Bigg]^2+2a_N(K).
\end{align}
\end{lemma}
Here the one-step mismatch is evaluated at states visited by the model. Its size is a separate alignment premise; contractivity bounds its accumulation. 
\begin{proof}
For almost every graph set $e_j=\|h_j^{\mathsf S}-h_j^\star\|_N$.
Inserting $U^\star_{G_N}(h_j^{\mathsf S})$ and using contractivity yields
\[
 e_{j+1}\le\|h_{j+1}^{\mathsf S}
             -U^\star_{G_N}(h_j^{\mathsf S})\|_N+\rho_N e_j
 \le\varepsilon_{\mathsf S}(N)+\rho_Ne_j.
\]
Both implementations in the lemma satisfy the assumed one-step bound.
The argument therefore applies to the shared processor with or without a controller, including iteration-dependent controller inputs.
Summing the geometric recurrence proves \eqref{eq:state-bound}.
The learned decoder satisfies
\[
 \|D_\eta(h_K^{\mathsf S})-D^\star(h_K^\star)\|_N
 \le L_De_K+\delta_{D,\mathsf S}(N).
\]
Add and subtract $D^\star(h_K^\star)$ in the total output error, use $\|x+y\|_N^2\le2\|x\|_N^2+2\|y\|_N^2$, and take graph expectations. The second expected squared norm is $a_N(K)$.
This proves Eq.~\ref{eq:alignment-risk}.
\end{proof}

The preceding lemma controls approximation at one graph size.
The next assumption controls its change across sizes.
Each $g\in\mathcal H_{\mathsf S}$ defines radius-$K$ predictions across sizes with shared parameters; write $g_N(G_N,v)$ for its node prediction at size $N$.

\begin{assumption}[Transfer of local approximation error]
\label{ass:size-transfer}
For $g\in\mathcal H_{\mathsf S}$, define
\[
 \calA_{N,K}(g)=\E\|g_N(G_N,v)-\widetilde Y_{N,K}\|^2.
\]
For the graph sizes considered, assume uniformly over the class that
\begin{equation}\label{eq:transfer}
 |\calA_{N,K}(g)-\calA_{N',K}(g)|
 \le\tau_{\mathsf S}(N,N',K),
\end{equation}
where $\tau_{\mathsf S}$ is a finite deterministic bound.
\end{assumption}
This assumption controls how local comparison mismatches change across graph sizes. Together with the estimation bound, it links fitting on adaptation graphs to deployment risk.
Sufficient distributional conditions are given in App.~\ref{app:transfer}.

\begin{lemma}[Estimation from adaptation graphs]\label{lem:P3}
At the prescribed depth, let adaptation use $m_{\mathrm{adapt}}$ i.i.d. graphs from $\mathcal P_{N_{\mathrm{adapt}}}$, independently of pretraining.
Suppose the graph-averaged squared-loss class of $\mathcal H_{\mathsf S}$ is bounded in $[0,B_{\mathrm{adapt}}]$ and has pseudo-dimension at most $p_{\mathsf S}\ge1$. Then
\begin{align}\label{eq:uniform-convergence}
 \E_{\mathcal D_{\mathrm{adapt}}}
 \sup_{g\in\mathcal H_{\mathsf S}}
 |\risk_{\mathrm{adapt}}(g)-\widehat\risk_{\mathrm{adapt}}(g)|
 \le\gamma_{\mathsf S}:=
 C_{\mathrm{gen}}B_{\mathrm{adapt}}
       \sqrt{\frac{p_{\mathsf S}}{m_{\mathrm{adapt}}}},
\end{align}
where $\risk_{\mathrm{adapt}}=\risk_{N_{\mathrm{adapt}}}$, $\widehat\risk_{\mathrm{adapt}}$ is its empirical version, and $C_{\mathrm{gen}}$ is universal. An empirical risk minimizer $f_{\mathsf S}\in\mathcal H_{\mathsf S}$ satisfies
\begin{equation}\label{eq:erm}
 \E_{\mathcal D_{\mathrm{adapt}}}\risk_{\mathrm{adapt}}(f_{\mathsf S})
 \le\inf_{g\in\mathcal H_{\mathsf S}}\risk_{\mathrm{adapt}}(g)
       +2\gamma_{\mathsf S}.
\end{equation}
\end{lemma}
Here $p_{\mathsf S}$ measures graph-loss complexity at the chosen depth.
The proof is in App.~\ref{app:estimation-proof}.
These primitives support the two parts of the following analysis: The optimal local prediction error establishes the information floor, while operator alignment and task-computation error control depth extension. Estimation and cross-size transfer then quantify adaptation as follows.

\subsection{Risk analysis of depth and adaptation}\label{sec:depth-adaptation}
We first compare fixed-depth prediction with iterative reuse, then analyze adaptation within each depth family. Throughout, $N$ is a generic graph size and $N_{\mathrm{train}}\le N_{\mathrm{adapt}}<N_{\mathrm{OOD}}$.
Conditional on pretraining, write $\overline{\mathcal R}_N(\mathsf S):=\E_{\mathcal D_{\mathrm{adapt}}}\risk_N(f_{\mathsf S})$ for scheme $\mathsf S$ at its specified depth, omitting the outer expectation for Fixed and Loop. Proofs and the common deployment bound (Lemma~\ref{lem:master}) are in App.~\ref{app:proofs}.

\subsubsection{Depth: information limits and iterative extension}
\paragraph{Fixed, Fixed-Tune, and Fixed-Steer.}
All three schemes retain $K_{\mathrm{fix}}$ one-hop processor applications.
Node-wise encoding, decoding, and controller interventions preserve this receptive radius.

\begin{proposition}[Information floor]\label{prop:ceiling}
For $\mathsf S\in\{\mathrm{Fixed},\mathrm{Fixed-Tune},\mathrm{Fixed-Steer}\}$, $K_{\mathrm{fix}}=O(1)$,
\begin{equation}\label{eq:ceiling}
 \err_N(K_{\mathrm{fix}})\le \overline{\mathcal R}_N(\mathsf S).
\end{equation}
If $L_\epsilon(N)\to\infty$, then $\overline{\mathcal R}_N(\mathsf S)>\epsilon^2$
eventually, including after adaptation.
\end{proposition}
This information floor is determined by the fixed receptive radius $K_{\mathrm{fix}}$ and applies to all three fixed-depth schemes. Adaptation changes the error above this floor by refining how the model uses the available neighborhood information.

\paragraph{Depth extension: Loop.}
Loop reuses the pretrained processor for a deployment-selected number of steps. Additional steps expand the available neighborhood, while their benefit depends on both the accuracy of the truncated task computation and the learned updates. Lemma~\ref{lem:accum} combines these effects into a prediction-risk bound.

\begin{proposition}[When iteration improves risk]\label{prop:iter}
Fix $\epsilon>0$. Suppose $L_\epsilon(N)>K_{\mathrm{fix}}$ and Loop is deployed at a depth where the conditions of
Lemma~\ref{lem:accum} hold and its risk bound is at most $\epsilon^2$.
Then, for Fixed, Fixed-Tune, and Fixed-Steer,
\begin{equation}\label{eq:separation}
 \overline{\mathcal R}_N(\mathrm{Loop}) \le\epsilon^2
 <\err_N(K_{\mathrm{fix}})\le \overline{\mathcal R}_N(\mathsf S).
\end{equation}
\end{proposition}
Thus aligned iteration can reach an accuracy beyond the information available to fixed-depth models, including adapted ones.
The proof and quantitative risk gap are in App.~\ref{app:iter}.
Sufficient task depths and a specific $\log$-depth case are given in App.~\ref{app:resolution-depth} and~\ref{app:geometric}.

\subsubsection{Adaptation: fine-tuning and freeze-and-steer}
\paragraph{Adaptation at a prescribed depth.}
The baseline $b(\mathsf S)$ is Fixed for Fixed-Tune/Fixed-Steer and
Loop for Loop-Tune/LFS. Each comparison uses the same depth with the baseline.
\begin{assumption}[Adaptation protocol]\label{ass:extrapolation}
Conditional on pretraining, fix the depth and class $\mathcal H_{\mathsf S}$ before sampling adaptation data.
The class contains its baseline and satisfies Lemma~\ref{lem:P3}; $f_{\mathsf S}$ is a measurable empirical risk minimizer.
The risks below are finite.
\end{assumption}
Define the achievable adaptation-distribution gain and its signed relative deployment shift:
\begin{align}
 \operatorname{Gap}_{\mathsf S}
 &:=\overline{\mathcal R}_{N_{\mathrm{adapt}}}(b(\mathsf S))
   -\inf_{g\in\mathcal H_{\mathsf S}}\risk_{N_{\mathrm{adapt}}}(g),
 \label{eq:source-gain}\\
 \Delta_{\mathsf S}
 &:=[\overline{\mathcal R}_{N_{\mathrm{OOD}}}(\mathsf S)
       -\overline{\mathcal R}_{N_{\mathrm{OOD}}}(b(\mathsf S))]-[\overline{\mathcal R}_{N_{\mathrm{adapt}}}(\mathsf S)
       -\overline{\mathcal R}_{N_{\mathrm{adapt}}}(b(\mathsf S))].
 \label{eq:relative-shift}
\end{align}
Here $\operatorname{Gap}_{\mathsf S}$ measures the improvement over the baseline achievable within the adaptation class at $N_{\mathrm{adapt}}$. The signed quantity $\Delta_{\mathsf S}$ measures how much the fitted model's improvement over that baseline decreases when moving to $N_{\mathrm{OOD}}$, with a negative value meaning a larger improvement at deployment. Together, these quantities separate adaptation capacity from the cross-size persistence of its benefit.

\begin{proposition}[Conditional adaptation gain]\label{prop:adapt}
Under Assumption~\ref{ass:extrapolation}, each adapted scheme satisfies
$\operatorname{Gap}_{\mathsf S}\ge0$
\begin{equation}\label{eq:adaptation}
 \operatorname{Gap}_{\mathsf S}-2\gamma_{\mathsf S}-\Delta_{\mathsf S}
 \le \overline{\mathcal R}_{N_{\mathrm{OOD}}}(b(\mathsf S))
       -\overline{\mathcal R}_{N_{\mathrm{OOD}}}(\mathsf S).
\end{equation}
\end{proposition}
Adaptation improves on its baseline when gain exceeds estimation cost and cross-size loss of benefit. The unknown shift $\Delta_{\mathsf S}$ can be bounded using local-error transfer and correction magnitude (App.~\ref{app:shift}), while changing depth at deployment additionally requires joint size-and-depth transfer control (App.~\ref{app:depth-schedules}).

\subsubsection{LFS: adaptation accuracy and parameter efficiency}
\label{sec:lfs-discussion}
LFS combines the information reach of looping with controller-and-readout adaptation. Its two comparisons answer different questions: comparison with Loop measures the benefit of adaptation, while comparison with Loop-Tune assesses accuracy relative to the trainable parameter budget.

\paragraph{Improvement over unadapted looping.}
Identity initialization places Loop within the LFS hypothesis class.
Proposition~\ref{prop:adapt} therefore gives the sufficient condition
\begin{equation}
    \operatorname{Gap}_{\mathrm{LFS}}
 >2\gamma_{\mathrm{LFS}}+\Delta_{\mathrm{LFS}}
\end{equation}
for strictly lower expected deployment risk than Loop. In words, the correction achievable by the controller and readout must exceed the finite-sample estimation cost and the loss of benefit across sizes.
Lemma~\ref{lem:accum} evaluates alignment of the complete controlled update, including the controller's effect on the frozen processor.

\paragraph{Trainable parameter budget.}
The counts for the six schemes are
\begin{center}
\begin{tabular}{@{}ll@{}}
\toprule
Schemes & Trainable parameters during adaptation\\
\midrule
Fixed, Loop & $0$\\
Fixed-Tune, Loop-Tune & $d_E+d_R+d_D$\\
Fixed-Steer, LFS & $d_\phi+d_D$\\
\bottomrule
\end{tabular}
\end{center}
Shared weights are counted once, independently of the number of iterations; $d_\phi$ includes any learned context embeddings.
Thus LFS uses the fraction
\[
 \frac{d_\phi+d_D}{d_E+d_R+d_D}
\]
of Loop-Tune's trainable budget. The saving is substantial when the controller is small relative to the frozen encoder and processor.
It reduces parameter-gradient and optimizer-state storage; total adaptation time and memory also depend on unrolled computation.

\paragraph{Adaptation data and size gaps.}
The comparison also depends on the adaptation-data budget and the size changes from pretraining to adaptation and deployment.
With few labeled adaptation graphs, LFS is attractive when its compact controller provides sufficient correction; a smaller loss-class complexity also reduces its estimation bound.
The ratio $N_{\mathrm{adapt}}/N_{\mathrm{train}}$ indicates how far adaptation extends beyond pretraining: LFS is suitable when the pretrained processor remains useful, whereas Loop-Tune offers broader updates when substantial correction is needed.
The ratio $N_{\mathrm{OOD}}/N_{\mathrm{adapt}}$ describes the remaining deployment size gap, over which either scheme's correction must transfer. These ratios contextualize the comparison rather than determine its outcome (App.~\ref{app:lfs-comparison}).


\section{Experimental Results}
\label{sec:experiments}
\renewcommand{\topfraction}{0.95}
\renewcommand{\bottomfraction}{0.6}
\renewcommand{\textfraction}{0.03}
\renewcommand{\floatpagefraction}{0.9}
\setcounter{topnumber}{3}
\setcounter{totalnumber}{4}

Our experiments answer three questions: (1) whether the task classes in Definition~\ref{def:classes} behave as predicted, (2) whether a looped model uses additional depth as graph size increases, and (3) when freeze-and-steer (LFS) is a better option than fine-tuning the looped model (Loop-Tune). We address the first two questions on synthetic tasks, and the third on both the synthetic tasks and a real-world chip-placement problem.
Our code, configuration files, and instructions to reproduce all
experiments are available at
\url{https://anonymous.4open.science/r/looped-gnn-size-gen-1F27}.
\subsection{Synthetic Dataset}
\label{sec:exp-synthetic}

\textbf{Tasks.}
We generate graphs from the random geometric model where $N$ points are placed
uniformly in the unit cube in $\mathbb{R}^d$, where two
points are connected when they are closer than a set radius $r$. The generated graphs keep the
average degree constant, so that the graph diameter grows as $N^{1/d}$.
On each graph we construct a node-level target in the iterative-operator
form of Definition~\ref{def:operator-form}. The node inputs $x_N$ are
encoded by a fixed linear map
$E^\star$. The update function is: 
\begin{equation*}
U_N(h) = (1-\alpha)\,\tilde A_N h + \alpha\, b(x_N),
\end{equation*}
with $\tilde A_N$ the symmetric normalized adjacency and $b$ a fixed
linear map, is applied $T^\star(N)$ times and then another linear map $D^\star$
produces the target values. This update contracts at rate
$\rho_N = 1-\alpha$ exactly, so $\alpha$ gives us direct
control over the contraction rate.

We construct four task families that are
identical except for how $T^\star$ grows with the graph size:
\textsc{const} keeps $T^\star = K_0$ fixed (class $\mathcal{F}_1$),
\textsc{log} grows it as $\log N$ (class $\mathcal{F}_2^{\log}$),
\textsc{poly} as $N^{1/2d}$, and \textsc{diam} as $N^{1/d}$, i.e.\
proportionally to the graph diameter (class $\mathcal{F}_3$). All four
families are normalized to the same horizon $T^\star = K_0 = 5$ at the
pretraining size $N_{\mathrm{ref}} = 200$.

\textbf{Protocol.}
Following the size ordering of Assumption~\ref{ass:extrapolation}, all
models are pretrained on small graphs ($N\in[100,200]$), adapted
on graphs a few times larger ($N\in[350,700]$), and then evaluated on
much larger held-out graphs ($N\in[10^3,2\times10^4]$). We compare the four schemes of
Sec.~\ref{sec:schemes}: \emph{Fixed} (implemented, as in standard
practice, with a distinct set of weights per layer; App.~\ref{app:exp-protocol}),
\emph{Loop}, \emph{Loop-Tune} and \emph{LFS} and put the results of the other two variants (\emph{Fixed-Tune} and
\emph{Fixed-Steer}) in App.~\ref{app:exp-tiers}. For the looped schemes, the deployment depth
at each test size is selected on the held-out validation graphs. We
further compare two pretraining objectives. Under \emph{target
supervision}, the standard machine learning end-to-end optimization approach, every iterate is
trained toward the final target only. Under \emph{operator-aligned}
supervision, each $k$-th iterate is trained toward the $k$-step partial
computation $D^\star U_N^{k} E^\star$. Graph sizes,
depths, parameter counts, and the depth-selection rule are given in
App.~\ref{app:exp-protocol}. Results average five random seeds if not otherwise stated.

\textbf{Fixed-radius error floor.}
We first verify that the constructed task families have the expected
dependency-radius behavior. As graph size increases, the fixed-radius
truncation error remains zero for \textsc{const} but increases for the
three growing-radius families, while its decay with radius is consistent
with the contraction analysis. Full results are in
App.~\ref{app:exp-primitives}.

\begin{figure*}[t]
  \centering
  \includegraphics[width=0.92\textwidth]{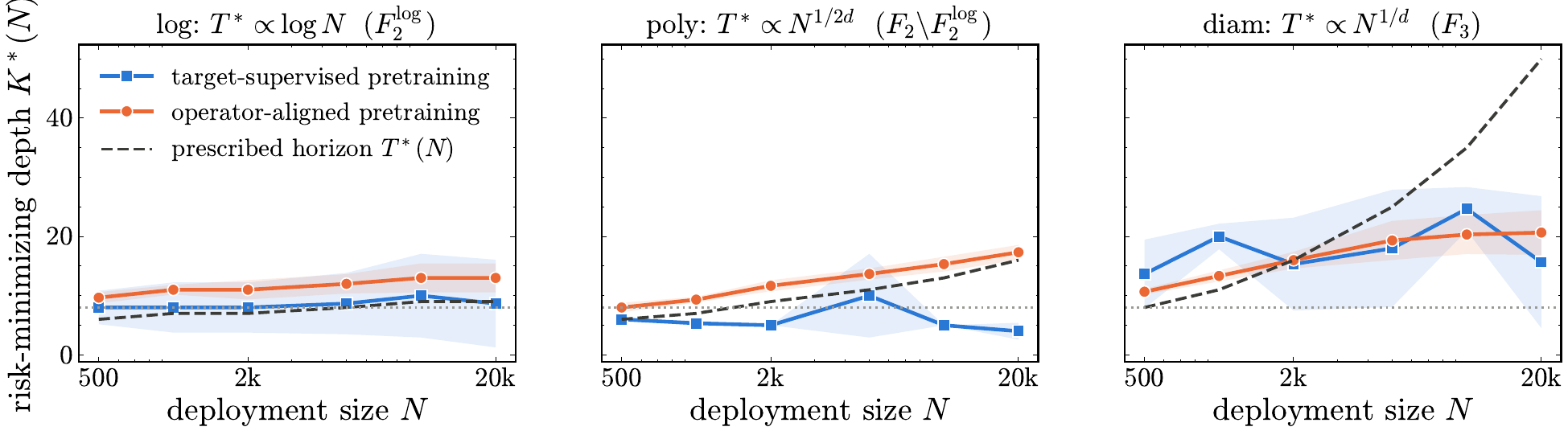}
  \caption{Risk-minimizing deployment depth as graph size increases.
  Operator-aligned pretraining tracks the prescribed task horizon
  $T^\star(N)$, whereas standard target supervision remains near the
  pretraining depth. Shading denotes variability across seeds; dotted
  line: the largest pretraining depth.}
  \label{fig:hier-kstar}
\end{figure*}

\textbf{Aligned processor.}
The horizon $T^\star(N)$ is known by construction, so we can check the
depth behavior that Proposition~\ref{prop:resolution-depth} suggests.
The proposition bounds the depth that the task computation needs, which
grows with the horizon $T^\star(N)$; if the learned processor
approximates the task update, the depth that minimizes test risk should
grow the same way.
Fig.~\ref{fig:hier-kstar} shows that whether this happens in fact depends on
how the model was pretrained. With operator-aligned pretraining, the optimal
depth grows with $N$ and tracks the prescribed horizon closely. With standard end-to-end target
supervision, the optimal depth stays at the pretraining depth in every
family, and risk rises beyond it. This shows that training end-to-end with the final target will make the processor only predicts in the pretraining horizon, not to approximate the
true update $U_N$. Therefore the ability to exactly extend depth at much larger graphs requires both looped architecture and 
aligned training, which is the role the
alignment error $\varepsilon_B$ plays in Lemma~\ref{lem:master}.
On \textsc{diam}, the chosen depth stops growing near $2.5\times$ the
pretraining depth instead of following the horizon to $T^\star = 50$.
The reason is that the learned update only approximates $U_N$, so every
additional step adds a small error, while the benefit of one more step
shrinks geometrically.

\begin{figure*}[t]
  \centering
  \includegraphics[width=0.75\textwidth]{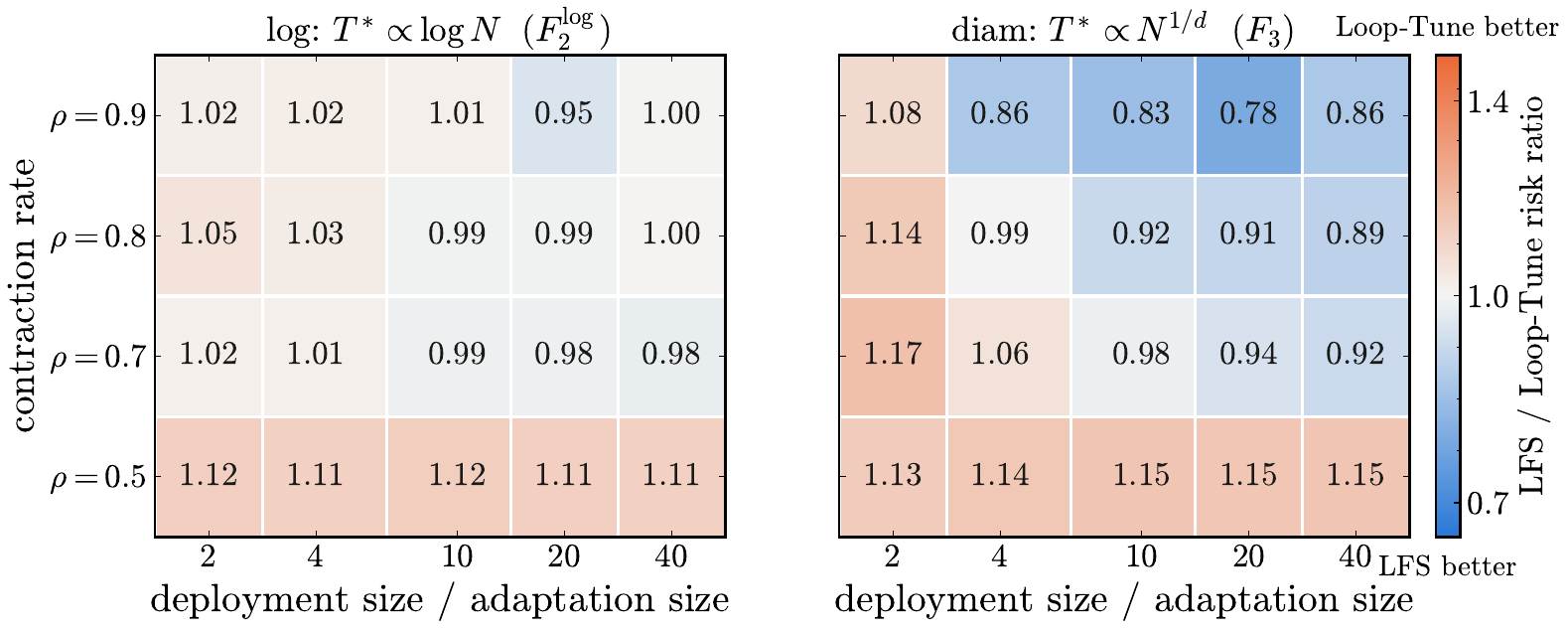}
  \caption{Risk ratio of LFS to Loop-Tune across contraction rates
  $\rho$ and deployment-to-adaptation size ratios, under
  target-supervised pretraining. Values $<1$
  indicate that LFS is more accurate. Results are averaged over five
  seeds.}
  \label{fig:hier-phase}
\end{figure*}

\textbf{Loop-Tune versus LFS.}
Fig.~\ref{fig:hier-phase} compares the two adaptation schemes across
the two quantities that the analysis of Sec.~\ref{sec:analysis}
identifies: how far the deployment size lies beyond the
adaptation sizes, and how slowly the update contracts. Neither scheme
dominates. Loop-Tune is more
accurate near the adaptation sizes and on tasks whose relevant
neighborhood is effectively small ($\rho = 0.5$, and the
$\mathcal{F}_1$ family at every $\rho$). LFS outperforms it on the two
families whose radius grows fastest (\textsc{poly} and \textsc{diam})
once the deployment size exceeds the adaptation size by roughly a factor
of ten at $\rho \ge 0.7$, by up to $22\%$ at $\rho = 0.9$; on
\textsc{log} the two stay within a few percent of each other.
Consistent with Proposition~\ref{prop:adapt}, the deciding factor is
whether adaptation would disturb the processor's behavior across sizes. 
When we instead pretrain and adapt with operator-aligned supervision, Loop-Tune wins
at every setting. The two main effects therefore do not occur together:
depth extension requires operator-aligned pretraining, and there
Loop-Tune is the better adaptation, while LFS's accuracy advantage
appears under target supervision, where the loop itself adds little once
the model is adapted (App.~\ref{app:exp-tiers}). What LFS keeps in both
regimes is the small adaptation budget. Its accuracy advantage also
returns when the adaptation data are scarce
(App.~\ref{app:exp-fsft}) or shifted: when we deliberately adapt on structurally
mismatched graphs, Loop-Tune degrades severely while LFS
is affected far less, winning in every seed at every size, by up to a
factor of $2.4$ (App.~\ref{app:exp-a14}). The complete study is in
App.~\ref{app:experiments-full}.

\subsection{Real-world Chip Placement}
\label{sec:exp-chip}

We now turn to a real-world task called very-large-scale
integration (VLSI) chip placement problem. In this problem, A circuit is given as a \emph{netlist},
a hypergraph $H = (V, E)$ whose nodes $V$ are the circuit cells and whose
hyperedges $E$ are the nets (wires) connecting them. Placement assigns
each cell a position on the two-dimensional die area, so that cells do
not overlap and connected cells stay close, minimizing objectives such as
total wirelength and congestion. We treat placement as node-level
regression: given the netlist, the model predicts each cell's
coordinates, with reference placements as supervision. A placement network with
a looped processor is \emph{pretrained on synthetic netlists} of at most
$500$ cells. It is then adapted on designs with $500$--$1{,}000$ cells
and evaluated on substantially larger designs, up to $2{,}400$--$2{,}800$
cells that largely beyond the adaptation range. We
consider two settings. In the first, adaptation and evaluation also use
synthetic netlists, so only the design size shifts. In the second, the
same pretrained model is adapted and evaluated on real
netlists from the 2023 MLCAD FPGA macro-placement contest
benchmark~\cite{bustany2023mlcad}, processed with the DE-HNN feature
pipeline~\cite{luo2024dehnn}. LFS trains the controller and the readout, which combined as about $2.9\%$ of the parameters.
Every adapted result is the mean of three independent adaptation runs
from the same pretrained network; the seed variation is at most $.05$
RMSE, far below every gap we discuss.

\begin{table}[!t]
  \centering
  \caption{Chip placement, physical RMSE (mean~$\pm$~sd over 3
  adaptation runs; lower is better; best per column in bold), with trainable parameters. Looped
  schemes deployed at $K{=}2$; `params' counts trained parameters;
  rightmost column: real MLCAD netlists.}
  \label{tab:chipgen}
  \scriptsize
  \setlength{\tabcolsep}{1.5pt}
  \resizebox{\columnwidth}{!}{%
  \begin{tabular}{lrccccc}
    \toprule
    & & \multicolumn{4}{c}{synthetic netlists, cells per design} & real \\
    \cmidrule(lr){3-6}
    scheme & params & 0.4--0.6k & 0.8--1k & 1.8--2k & 2.4--2.8k & MLCAD \\
    \midrule
    Fixed       & 0 & 1.51 & 3.00 & 6.17 & 8.50 & 7.18 \\
    Fixed-Tune  & 5.85M & 1.48{\tiny$\pm$.00} & 2.81{\tiny$\pm$.00} & 5.59{\tiny$\pm$.00} & 7.54{\tiny$\pm$.01} & 6.78{\tiny$\pm$.00} \\
    Loop-Tune   & 5.85M & \textbf{1.11}{\tiny$\pm$.01} & \textbf{1.86}{\tiny$\pm$.00} & \textbf{3.69}{\tiny$\pm$.02} & \textbf{5.30}{\tiny$\pm$.04} & \textbf{6.01}{\tiny$\pm$.05} \\
    LFS         & \textbf{0.17M} & 1.27{\tiny$\pm$.01} & 2.06{\tiny$\pm$.01} & 3.95{\tiny$\pm$.02} & 5.61{\tiny$\pm$.01} & 6.24{\tiny$\pm$.02} \\
    \bottomrule
  \end{tabular}}
\end{table}

Tab.~\ref{tab:chipgen} reports the results. The loop matters most: on
the largest synthetic designs, the two looped adaptation schemes cut the
error of Fixed-Tune by $25$--$30\%$ (Loop-Tune $5.30$ and LFS
$5.61$ versus $7.54$). 
On the synthetic designs, where
the adaptation and deployment data differ only in size, Loop-Tune is the most accurate at every size while LFS stays close using only $2.9\%$ of its trainable parameters. 
The same ordering holds on real MLCAD netlists under additional distribution shift: Loop-Tune achieves $6.01$ and LFS $6.24$,
both below Fixed-Tune's $6.78$. 
Freezing therefore beats fixed-depth fine-tuning under
the shift.
In practice, adaptation must balance predictive accuracy with
resource costs. LFS offers a favorable accuracy--parameter trade-off:
its error is within $4\%$ of Loop-Tune's while reducing the trainable
parameter budget by $97.1\%$. This makes LFS an attractive option
when adaptation resources are limited.

\section{Conclusion}
Looped GNNs enable graph-size generalization by reusing learned local updates beyond training depth. Through six schemes, we distinguish the contributions of depth extension and adaptation, and introduce LFS to steer iterative reuse with fewer trainable parameters.
Our theory and experiments identify conditions supporting these benefits. The analysis focuses on contractive dynamics with controlled alignment and cross-size transfer. Future work will broaden these conditions, develop label-free deployment-depth selection, and evaluate larger chip designs.

\clearpage

\bibliography{references}
\bibliographystyle{ieeetr}

\clearpage
\appendix
\thispagestyle{empty}

\onecolumn

\section{Examples of dependency-radius regimes}\label{app:regime-examples}

\subsection{What the classification measures}
For a uniformly sampled node $v$ in $G_N\sim\mathcal P_N$, the
effective radius $L_\epsilon(N)$ is the smallest integer $k$ for which
the optimal squared prediction error using $B_k(G_N,v)$ is at most
$\epsilon^2$. In the notation of the main analysis,
\[
 \err_N(k)=\E\!\left[\left\|f_N^\star(G_N,v)
 -\E[f_N^\star(G_N,v)\mid\mathscr B_{N,k}]\right\|^2\right].
\]
We compare this radius with a deterministic diameter scale $D_N$:
bounded radius gives $\mathcal F_1^{(\epsilon)}$;
$L_\epsilon(N)\to\infty$ with $L_\epsilon(N)=o(D_N)$ gives
$\mathcal F_2^{(\epsilon)}$; and
$L_\epsilon(N)=\Theta(D_N)$ gives $\mathcal F_3^{(\epsilon)}$.
Thus the classification concerns the information needed at a fixed
accuracy, rather than the number of steps in a chosen implementation.
All examples below assume finite second moments. Neighborhoods contain
the induced edges and all specified node and edge attributes, with
no additional global summaries. The predictor may depend on $N$,
as permitted by the size-indexed definition.

\subsection{Bounded-local targets: $\mathcal F_1^{(\epsilon)}$}
\paragraph{Degree and local clustering.}
The degree of a node is determined by its one-hop neighborhood, so
$\err_N(1)=0$ and $L_\epsilon(N)\le1$. Its local clustering coefficient
is the number of edges between its neighbors divided by the number of
possible such edges, with value zero for degree less than two.
It is also determined by the induced one-hop neighborhood.
If a neighborhood representation omits edges between boundary nodes,
one extra hop suffices; either convention gives bounded radius.

\paragraph{Fixed-hop label counts.}
For a fixed integer $r$ and a node label $\lambda_u$, the target
\[
 f_N^\star(G_N,v)=\sum_{u\in B_r(G_N,v)}
                    \mathbf 1\{\lambda_u=1\}
\]
has $L_\epsilon(N)\le r$ for every $N$. Its output magnitude may grow
on high-degree graphs, but its required radius stays bounded.
The effective radius can be smaller than $r$ when the distribution
makes some labels predictable or their effect falls below tolerance.

\subsection{Growing sub-diameter targets: $\mathcal F_2^{(\epsilon)}$}
\paragraph{An explicit neighborhood statistic.}
Let $G_N$ have the topology of the cycle $C_N$ and independent
standard Gaussian node attributes $X_u$. Its diameter is
$D_N=\lfloor N/2\rfloor$. Choose an integer radius $k_N\to\infty$
with $k_N=o(N)$, and define the variance-normalized local sum
\[
 f_N^\star(G_N,v)=\frac{1}{\sqrt{2k_N+1}}
                      \sum_{u\in B_{k_N}(G_N,v)}X_u.
\]
For sufficiently large $N$, the ball contains $2k_N+1$ nodes.
Conditional on $B_k(G_N,v)$, the attributes outside the observed
ball are still independent with mean zero and variance one.
The optimal predictor sums the observed terms, and hence
\[
 \err_N(k)=\left(1-\frac{2k+1}{2k_N+1}\right)_+,
 \qquad k\ge0,
\]
where $(x)_+=\max\{x,0\}$. For any fixed $0<\epsilon<1$,
\[
 L_\epsilon(N)=
 \max\!\left\{0,
 \left\lceil\frac{(1-\epsilon^2)(2k_N+1)-1}{2}\right\rceil\right\}
 =\Theta(k_N).
\]
Thus this family belongs to $\mathcal F_2^{(\epsilon)}$.
For example, $k_N=\lceil\log N\rceil$ gives logarithmic reach,
whereas $k_N=\lfloor N^{1/2}\rfloor$ gives polynomial but
sub-diameter reach. Both have unit output variance: their increasing
radius is not an artifact of growing output magnitude.

\paragraph{Finite-horizon local dynamics.}
A node-wise encoder and decoder with $T_N$ one-hop updates give
$L_\epsilon(N)\le T_N$ for an exactly represented finite-horizon
target. To classify such a target as growing sub-diameter, one must
also establish a growing information requirement. For example, if
$T_N\to\infty$, $T_N=o(D_N)$, and for a constant $c\in(0,1)$,
\[
 \err_N(\lfloor cT_N\rfloor)>\epsilon^2
 \quad\text{for all sufficiently large }N,
\]
then $cT_N<L_\epsilon(N)\le T_N$, giving
$L_\epsilon(N)=\Theta(T_N)$. This lower-error condition expresses
that information arriving late in the computation matters at the
chosen accuracy. A growing simulation horizon alone supplies only
the upper bound.

\subsection{Diameter-scale targets: $\mathcal F_3^{(\epsilon)}$}
\paragraph{A global aggregate broadcast to every node.}
On the same attributed cycle, consider
\[
 f_N^\star(G_N,v)=N^{-1/2}\sum_{u\in V(G_N)}X_u.
\]
The value is identical at every node and has variance one. A radius-$k$
ball reveals $\min\{N,2k+1\}$ attributes, so independence gives
\[
 \err_N(k)=1-\frac{\min\{N,2k+1\}}{N}.
\]
For fixed $0<\epsilon<1$,
\[
 L_\epsilon(N)=
 \max\!\left\{0,
 \left\lceil\frac{N(1-\epsilon^2)-1}{2}\right\rceil\right\}
 =\Theta(N)=\Theta(D_N).
\]
Unlike the preceding local sum, this target requires observing a
nonvanishing fraction of the graph to reach the tolerance.

\paragraph{Distance from a marked source.}
Let $G_N$ be an unweighted cycle with a source $s$ sampled uniformly
and marked by a binary node attribute. Take
$f_N^\star(G_N,v)=\operatorname{dist}(s,v)$, and give the predictor
no coordinates or other attributes that reveal the source location.
For $k\le\lfloor N/8\rfloor$, the source lies outside $B_k(v)$
with probability bounded away from zero. On that event the observed
rooted neighborhood contains no source mark and leaves the source
uniform over the unobserved nodes. Its distance to $v$ ranges from
$k+1$ to $\lfloor N/2\rfloor$, with two nodes per distance except
possibly at the antipode. Its conditional variance is therefore
$\Theta(N^2)$, uniformly over these $k$. Consequently
\[
 \err_N(\lfloor N/8\rfloor)\ge cN^2
\]
for some $c>0$ and all sufficiently large $N$. Observing radius
$D_N$ reveals the source and its exact distance, giving
$L_\epsilon(N)=\Theta(D_N)$ for every fixed $\epsilon>0$.
The key is unresolved source location, not merely a large typical
source--query distance. Coordinates or other informative attributes
can change this conclusion.

\paragraph{Node eccentricity.}
The eccentricity $\max_u\operatorname{dist}(v,u)$ has a global
definition, but its regime depends on the graph distribution. On
$C_N$ it equals $\lfloor N/2\rfloor$ at every node; a size-aware
constant predictor has zero error even at radius zero.
For a distribution supported on graphs of diameter at most $CD_N$,
a sufficient diameter-scale condition is
\[
 \err_N(\lfloor cD_N\rfloor)>\epsilon^2
\]
eventually for some constants $C>c>0$. The full-graph upper bound
and this conditional-uncertainty lower bound then imply
$L_\epsilon(N)=\Theta(D_N)$. On general random graph families,
diameter concentration alone does not replace the needed error control
on rare graphs.

\subsection{Normalization, tolerance, and iterative structure}
\paragraph{A global mean can be bounded-local.}
Replacing the variance-normalized global sum above by the mean
$N^{-1}\sum_u X_u$ gives
\[
 \err_N(0)=\frac{N-1}{N^2}\longrightarrow0.
\]
Thus $L_\epsilon(N)=0$ eventually for every fixed $\epsilon>0$,
even though the exact target uses every node. This illustrates why
the target normalization and tolerance must be specified along with
the graph distribution.

\paragraph{Damped dynamics can have bounded effective radius.}
On a cycle, let $P$ be the neighbor-averaging operator and consider
\[
 h_{t+1}=\alpha Ph_t+(1-\alpha)X,\qquad h_0=0,
 \qquad 0<\alpha<1.
\]
Here $X$ is a fixed input attribute accessible at each node, and
$\|P\|\le1$ in the per-node norm. For the equilibrium target
$h_\infty=(1-\alpha)\sum_{j\ge0}\alpha^jP^jX$, the tail satisfies
\[
 a_N(k)=\E\|h_\infty-h_k\|_N^2
 \le\alpha^{2k}\E\|X\|_N^2=\alpha^{2k}.
\]
Consequently $L_\epsilon(N)=O(1)$ for fixed tolerance, despite the
infinite task horizon. A growing effective radius requires further
size dependence, such as damping approaching one, or a different
signal scale or accuracy requirement, together with an appropriate
information lower bound.

\paragraph{Scope.}
These examples illustrate the distribution-dependent radius regimes.
Membership in a regime does not itself assert a contractive iterative
representation with fixed hidden dimension. In particular, the explicit
neighborhood sums and distance example establish information requirements;
the operator-form and stability assumptions used in the performance
analysis must be checked separately for a chosen implementation.

\section{Proofs of the main statements}\label{app:proofs}

\subsection{Information projection}\label{app:information}
\begin{proof}[Justification of Eq.~\ref{eq:projection}]
Conditional expectation is the orthogonal projection onto the
square-integrable $\mathscr B_{N,k}$-measurable vectors.
For any such predictor $g$, the cross term in
$g-Y_N=(g-m_{N,k}^\star)+(m_{N,k}^\star-Y_N)$ has expectation zero.
Expanding the squared norm gives Eq.~\ref{eq:projection}.
Taking $g=\widetilde Y_{N,k}$ gives
\begin{equation}\label{eq:two-errors}
 a_N(k)=\err_N(k)+\E\|\widetilde Y_{N,k}-m_{N,k}^\star\|^2
 \ge\err_N(k),
\end{equation}
since node-wise encoding and decoding with $k$ one-hop updates
produce a radius-$k$ local prediction.
\end{proof}

\subsection{Proof of Lemma~\ref{lem:P3}: estimation from adaptation graphs}
\label{app:estimation-proof}
\begin{proof}[Proof of Lemma~\ref{lem:P3}]
Condition on the pretrained model and fix a scheme $\mathsf S$ and
depth $K$. The hypothesis class is consequently fixed before sampling
the adaptation graphs. Abbreviate
\[
 n=N_{\mathrm{adapt}},\quad m=m_{\mathrm{adapt}},\quad
 B=B_{\mathrm{adapt}},\quad p=p_{\mathsf S}\ge1,
 \quad\mathcal H=\mathcal H_{\mathsf S}.
\]
Let $G_1,\ldots,G_m$ be i.i.d. from $\mathcal P_n$ and use the
measurability conditions on classes and suprema stated in the setup.
If $B=0$, all losses vanish and both claims are immediate; suppose
$B>0$ below.

\paragraph{Step 1: reduce to a bounded graph-level function class.}
For $g\in\mathcal H$, define
\[
 \ell_g(G):=\frac1n\sum_{v\in V(G)}
       \|g_n(G,v)-f_n^\star(G,v)\|^2.
\]
Then $0\le\ell_g\le B$, and
\[
 \risk_n(g)=\E_G\ell_g(G),\qquad
 \widehat\risk_n(g)=\frac1m\sum_{i=1}^m\ell_g(G_i).
\]
Set $\mathcal F=\{\ell_g/B:g\in\mathcal H\}$; this is a
$[0,1]$-valued class of pseudo-dimension at most $p$, since positive
scaling preserves pseudo-dimension. Write $Pf=\E_G f(G)$ and
$P_mf=m^{-1}\sum_i f(G_i)$. It suffices to show
\begin{equation}\label{eq:est-normalized-goal}
 \E\sup_{f\in\mathcal F}|Pf-P_mf|\le C\sqrt{p/m}.
\end{equation}
Each observation here is an entire graph. Dependence among nodes of
the same graph is therefore allowed.

\paragraph{Step 2: symmetrize the uniform deviation.}
Let $G_1',\ldots,G_m'$ be an independent sample from the same
distribution, with empirical measure $P_m'$. Jensen's inequality
and the supremum give
\[
 \E\sup_{f\in\mathcal F}|Pf-P_mf|
 \le\E\sup_{f\in\mathcal F}|P_m'f-P_mf|.
\]
Let $\sigma_1,\ldots,\sigma_m$ be independent Rademacher signs,
independent of both graph samples. Exchangeability of each pair
$(G_i,G_i')$ and the triangle inequality imply
\begin{align}
 \E\sup_{f\in\mathcal F}|Pf-P_mf|
 &\le\E\sup_{f\in\mathcal F}
       \left|\frac1m\sum_{i=1}^m
          \sigma_i\bigl(f(G_i')-f(G_i)\bigr)\right| \le2\E_{G,\sigma}\sup_{f\in\mathcal F}
       \left|\frac1m\sum_{i=1}^m\sigma_i f(G_i)\right|.
       \label{eq:est-symmetrization}
\end{align}
To handle the absolute value, use
$\mathcal F_{\pm}=\mathcal F\cup(-\mathcal F)\cup\{0\}$.
The last supremum equals
$\sup_{f\in\mathcal F_{\pm}}m^{-1}\sum_i\sigma_i f(G_i)$.

\paragraph{Step 3: control complexity using pseudo-dimension.}
For a fixed graph sample, equip function values with the empirical
$L^2$ pseudometric
\[
 d_m(f,h)=\left(\frac1m\sum_{i=1}^m
                         (f(G_i)-h(G_i))^2\right)^{1/2}.
\]
Let $\mathcal N(t,\mathcal F,L^2(Q))$ denote the covering number at
radius $t$. The uniform VC-subgraph covering bound for a bounded class
of pseudo-dimension at most $p$ gives, for every probability measure
$Q$ and $0<t\le1$,
\begin{equation}\label{eq:est-entropy}
 \log\mathcal N(t,\mathcal F,L^2(Q))
 \le C_0p\log(C_1/t),
\end{equation}
with universal constants. This is the standard VC-subgraph entropy
bound~\cite[Theorem~8.2]{wellner2005empirical}.
Since
\[
 \mathcal N(t,\mathcal F_{\pm},d_m)
 \le2\mathcal N(t,\mathcal F,d_m)+1,
\]
the same logarithmic entropy bound, with adjusted universal constants,
holds for $\mathcal F_{\pm}$ for $0<t\le2$. Here $p\ge1$ absorbs
the additional constant and monotonicity extends the bound beyond
$t=1$.

\paragraph{Step 4: integrate the entropy bound.}
Conditionally on the sample, the process
$X_f=m^{-1}\sum_i\sigma_i f(G_i)$ is centered, $X_0=0$, and satisfies
\[
 \E_\sigma\exp\{\lambda(X_f-X_h)\}
 \le\exp\left\{\frac{\lambda^2d_m(f,h)^2}{2m}\right\},
 \qquad\lambda\in\R.
\]
Indeed, the left side is a product of hyperbolic cosines, and
$\cosh(x)\le e^{x^2/2}$. Thus the increments are sub-Gaussian in
the metric $d_m/\sqrt m$. Dudley's entropy bound, using
$\operatorname{diam}(\mathcal F_{\pm},d_m)\le2$, gives
\begin{align*}
 \E_\sigma\sup_{f\in\mathcal F_{\pm}}X_f
 &\le\frac{C_2}{\sqrt m}\int_0^2
       \sqrt{\log\mathcal N(t,\mathcal F_{\pm},d_m)}\,dt\\
 &\le C_3\sqrt{\frac p m}\int_0^2
       \sqrt{\log(C_4/t)}\,dt
 \le C_5\sqrt{\frac p m},
\end{align*}
where $C_4>2$ can be chosen universal. The integral is finite:
with $u=\log(C_4/t)$ it equals
$C_4\int_{\log(C_4/2)}^\infty\sqrt u\,e^{-u}\,du$.
This bound is uniform over the graph sample. Combining it with
\eqref{eq:est-symmetrization} proves \eqref{eq:est-normalized-goal}.
Multiplying by $B$ yields
\[
 \E\sup_{g\in\mathcal H}
 |\risk_n(g)-\widehat\risk_n(g)|
 \le C_{\mathrm{gen}}B\sqrt{p/m},
\]
which is the asserted uniform-deviation bound.

\paragraph{Step 5: derive the empirical-minimization guarantee.}
Let $Z(\mathcal D)=\sup_{g\in\mathcal H}
|\risk_n(g)-\widehat\risk_n(g)|$ and let $f_{\mathsf S}$ be
an empirical risk minimizer. For every sample and every fixed
comparator $g\in\mathcal H$,
\begin{align*}
 \risk_n(f_{\mathsf S})
 &\le\widehat\risk_n(f_{\mathsf S})+Z(\mathcal D)\\
 &\le\widehat\risk_n(g)+Z(\mathcal D)
 \le\risk_n(g)+2Z(\mathcal D).
\end{align*}
Taking the infimum over comparators and then expectation gives
\[
 \E_{\mathcal D}\risk_n(f_{\mathsf S})
 \le\inf_{g\in\mathcal H_{\mathsf S}}\risk_n(g)
       +2C_{\mathrm{gen}}B_{\mathrm{adapt}}
             \sqrt{\frac{p_{\mathsf S}}{m_{\mathrm{adapt}}}},
\]
which proves the second claim. 
\end{proof}

\subsection{Deployment risk}\label{app:master}
For $\mathsf S\in\{\mathrm{Fixed},\mathrm{Loop}\}$, set $n_{\mathsf S}=N_{\mathrm{train}}$,
$\gamma_{\mathsf S}=0$, and let $\mathcal H_{\mathsf S}$ be the singleton
containing the pretrained predictor at its prescribed depth.
For $\mathsf S\in\{\text{Fixed-Tune},\text{Fixed-Steer},\text{Loop-Tune},\mathrm{LFS}\}$, set
$n_{\mathsf S}=N_{\mathrm{adapt}}$ and use the adaptation class and
$\gamma_{\mathsf S}$ from Lemma~\ref{lem:P3}.
The zero adaptation penalty leaves pretraining error in the source risk.
For all schemes, $K$ denotes their common source/deployment depth;
it is $K_{\mathrm{fix}}$ for the fixed-depth family.

\begin{lemma}[Deployment risk]\label{lem:master}
Under Assumption~\ref{ass:size-transfer}, suppose an adapted scheme
is an empirical risk minimizer satisfying Lemma~\ref{lem:P3}.
For each scheme at depth $K$,
\begin{align}\label{eq:master}
 \err_N(K)\le\E\risk_N(f_{\mathsf S})
 &\lesssim\inf_{g\in\mathcal H_{\mathsf S}}\risk_{n_{\mathsf S}}(g)
   +\gamma_{\mathsf S}+\tau_{\mathsf S}(N,n_{\mathsf S},K)+a_{n_{\mathsf S}}(K)+a_N(K).
\end{align}
The expectation is over adaptation data when used; the implicit
constant is universal. The same formula applies directly to the
singleton classes of Fixed and Loop.
\end{lemma}
The upper bound separates source approximation, estimation, transfer,
and task-computation errors. Lemma~\ref{lem:accum} can control source
approximation using a comparator in the class. This is an alternative
risk certificate, not an additional alignment penalty to add twice.

\begin{proof}[Proof of Lemma~\ref{lem:master}]
Fix $n=n_{\mathsf S}$. For each predictor in the class,
\begin{align*}
 \risk_N(g)
 &\le2\calA_{N,K}(g)+2a_N(K)\\
 &\le2\calA_{n,K}(g)+2\tau_{\mathsf S}(N,n,K)+2a_N(K)\\
 &\le4\risk_n(g)+4a_n(K)+2\tau_{\mathsf S}(N,n,K)+2a_N(K).
\end{align*}
Uniformity permits substitution of the fitted predictor. Estimation gives
\begin{align*}
 \E\risk_N(f_{\mathsf S})
 \le{}&4\inf_{g\in\mathcal H_{\mathsf S}}\risk_n(g)
       +8\gamma_{\mathsf S}+2\tau_{\mathsf S}(N,n,K)+4a_n(K)+2a_N(K).
\end{align*}
For singleton classes apply the deterministic inequality directly.
The lower bound follows from Eq.~\ref{eq:projection}, conditional
on each fitted predictor, and independence of the deployment graph.
The term $a_n(K)$ remains necessary when the common depth exceeds the source task's horizon; it then measures over-iteration.
\end{proof}

\subsection{Fixed-radius limitation}\label{app:fixed-radius}
\begin{proof}[Proof of Proposition~\ref{prop:ceiling}]
Condition on pretraining and, when used, the adaptation sample.
The fitted weights are now fixed. By induction, after $j$ one-hop
updates each node's state depends only on its rooted $j$-hop
neighborhood: the encoder is node-wise, each processor application
adds at most one hop, and the node-wise controller adds none.
Locally available history, iteration indices, and the deterministic
graph size do not enlarge this radius. The node-wise readout therefore
gives a $\mathscr B_{N,K_{\mathrm{fix}}}$-measurable prediction.

This argument applies to Fixed, to all-parameter Fixed-Tune, and to
controller-only Fixed-Steer, regardless of how their weights were fitted.
On a fresh independent graph the projection identity
Eq.~\ref{eq:projection} applies conditionally on those weights.
Averaging over the adaptation sample gives
\[
 \overline{\mathcal R}_{N}(\mathsf S)=\err_N(K_{\mathrm{fix}})
 +\E_{\mathcal D_{\mathrm{adapt}},G_N,v}
   \|f_{\mathsf S}(G_N,v)-m_{N,K_{\mathrm{fix}}}^\star\|^2.
\]
The second term is nonnegative, proving Eq.~\ref{eq:ceiling}.
For Fixed the outer adaptation expectation is omitted.
Since $L_\epsilon(N)\to\infty$, eventually every bounded
$K_{\mathrm{fix}}$ is smaller than $L_\epsilon(N)$ and therefore
$\err_N(K_{\mathrm{fix}})>\epsilon^2$.
\end{proof}

\subsection{Sufficient task-computation depth}\label{app:resolution-depth}
\begin{proposition}[Sufficient resolution depth]\label{prop:resolution-depth}
Under Assumption~\ref{ass:contraction}, suppose $D^\star$ is
$L_D^\star$-Lipschitz and
\begin{equation}\label{eq:amplitude}
 A_N:=\frac{L_D^\star}{1-\rho_N}
 \left(\E_G\|U^\star_{G_N}(h_0^\star)-h_0^\star\|_N^2\right)^{1/2}
 <\infty.
\end{equation}
For $0\le k\le T_N^\star$,
$\err_N(k)\le a_N(k)\le A_N^2\rho_N^{2k}$.
For fixed $\epsilon>0$, define
\[
 K_{\mathrm{res}}(N,\epsilon)
 :=\min\{k\in\mathbb N_0:k\le T_N^\star,\ a_N(k)\le\epsilon^2\}.
\]
Then
\begin{equation}\label{eq:sufficient-depth}
 L_\epsilon(N)\le K_{\mathrm{res}}(N,\epsilon)
 \le\min\left\{T_N^\star,
 \left\lceil\frac{\log\max\{1,A_N/\epsilon\}}{|\log\rho_N|}
 \right\rceil\right\}.
\end{equation}
\end{proposition}
For equilibrium targets take $T_N^\star=\infty$ and assume the target
is the decoded $L^2$ limit of the task states. This proposition bounds
task-computation depth, not the depth minimizing implemented-model risk.

\begin{proof}[Proof of Proposition~\ref{prop:resolution-depth}]
Contractivity gives
\[
 \|h_{j+1}^\star-h_j^\star\|_N
 \le\rho_N^j\|h_1^\star-h_0^\star\|_N.
\]
Sum the increments from $j=k$ to $T_N^\star-1$ and apply the decoder's
Lipschitz bound and the triangle inequality in $L^2$. Thus
\begin{equation}\label{eq:truncation-envelope}
 \sqrt{a_N(k)}
 \le A_N\rho_N^k(1-\rho_N^{T_N^\star-k})
 \le A_N\rho_N^k.
\end{equation}
The rounded threshold in Eq.~\ref{eq:sufficient-depth} makes the last
bound at most $\epsilon$. If it exceeds the task horizon, use
$a_N(T_N^\star)=0$ instead. Locality gives
$\err_N(k)\le a_N(k)$, hence $L_\epsilon(N)\le K_{\mathrm{res}}$.
The case $A_N=0$ gives sufficient depth zero.
For equilibrium targets, assume the task states converge in $L^2$;
passing to the limit gives the result with $T_N^\star=\infty$ and
the factor $1-\rho_N^{T_N^\star-k}$ replaced by $1$.
\end{proof}

\subsection{Iteration versus fixed-depth adaptation}\label{app:iter}
\begin{proof}[Proof of Proposition~\ref{prop:iter}]
Conditional on pretraining, Loop changes only the number of applications
of $R_{\theta_0,G_N}$; it retains $E_0$ and $D_{\eta_0}$ and has no
adaptation-sample randomness. Let $K_N$ denote its prescribed deployment
depth and let $\mathcal U_{\mathrm{Loop},N}(K_N)$ be the right-hand
side of Eq.~\ref{eq:alignment-risk} specialized to Loop. By hypothesis,
\[
 \overline{\mathcal R}_{N}(\mathrm{Loop})\le\mathcal U_{\mathrm{Loop},N}(K_N)
 \le\epsilon^2.
\]
Proposition~\ref{prop:ceiling} gives
$\overline{\mathcal R}_{N}(\mathsf S)\ge\err_N(K_{\mathrm{fix}})$ for each of the three
fixed-depth schemes, including the adapted ones.
By the definition of $L_\epsilon(N)$, the condition
$K_{\mathrm{fix}}<L_\epsilon(N)$ implies
$\err_N(K_{\mathrm{fix}})>\epsilon^2$. Combining these inequalities proves
Eq.~\ref{eq:separation}.
\end{proof}

\paragraph{Quantitative comparison.}
More generally, whenever Lemma~\ref{lem:accum} applies to Loop at $K_N$,
subtracting its upper bound from the fixed-radius lower bound yields
\begin{equation}\label{eq:quantitative-separation}
 \overline{\mathcal R}_{N}(\mathsf S)-\overline{\mathcal R}_{N}(\mathrm{Loop})
 \ge\err_N(K_{\mathrm{fix}})-\mathcal U_{\mathrm{Loop},N}(K_N).
\end{equation}
For $K_{\mathrm{fix}}=O(1)$, if $L_\epsilon(N)\to\infty$ and
$\mathcal U_{\mathrm{Loop},N}(K_N)\le\epsilon^2$ for all sufficiently
large $N$, Proposition~\ref{prop:iter} gives eventual strict improvement.
The full risk bound includes encoder, accumulated update-alignment,
decoder, and task-computation errors. Reaching a small task-computation
error alone is insufficient; uniformly bounded per-step errors also
need not ensure the required accuracy, especially when $\rho_N\to1$.

\subsection{Conditional gains for all adaptation variants}
\label{app:adaptation-proof}
\begin{proof}[Proof of Proposition~\ref{prop:adapt}]
Fix $\mathsf S\in\{\text{Fixed-Tune},\text{Fixed-Steer},\text{Loop-Tune},\mathrm{LFS}\}$ and condition on pretraining.
The baseline $f_{b(\mathsf S)}$, the depth, and
$\mathcal H_{\mathsf S}$ are then fixed independently of the adaptation
sample $\mathcal D$.
For Fixed-Tune, retaining the pretrained encoder, processor, and
readout recovers Fixed; for Fixed-Steer, $C_0=\mathrm{id}$ does so.
For Loop-Tune, retaining the pretrained processor and readout
recovers Loop; for LFS, $C_0=\mathrm{id}$ and $\eta=\eta_0$ do so.
These choices are admissible by the baseline-inclusion assumption,
at the same prescribed depth and across the graph sizes being compared.
Consequently
\[
 \inf_{g\in\mathcal H_{\mathsf S}}\risk_{N_{\mathrm{adapt}}}(g)
 \le \overline{\mathcal R}_{N_{\mathrm{adapt}}}(b(\mathsf S)),
\]
so $G_{\mathsf S}\ge0$.
Put
\[
 Z_{\mathsf S}(\mathcal D)
 :=\sup_{g\in\mathcal H_{\mathsf S}}
       |\risk_{N_{\mathrm{adapt}}}(g)
         -\widehat\risk_{N_{\mathrm{adapt}}}(g)|.
\]
Lemma~\ref{lem:P3} gives $\E [Z_{\mathsf S}]\le\gamma_{\mathsf S}$.
For every comparator $g\in\mathcal H_{\mathsf S}$, empirical
minimization implies
\begin{align*}
 \risk_{N_{\mathrm{adapt}}}(f_{\mathsf S})
 &\le\widehat\risk_{N_{\mathrm{adapt}}}(f_{\mathsf S})
        +Z_{\mathsf S}\\
 &\le\widehat\risk_{N_{\mathrm{adapt}}}(g)+Z_{\mathsf S}
 \le\risk_{N_{\mathrm{adapt}}}(g)+2Z_{\mathsf S}.
\end{align*}
Taking the infimum and expectation gives
$\overline{\mathcal R}_{N_{\mathrm{adapt}}}(\mathsf S)\le
\inf_{g\in\mathcal H_{\mathsf S}}\risk_{N_{\mathrm{adapt}}}(g)
 +2\gamma_{\mathsf S}$.
The definition of $\Delta_{\mathsf S}$ gives the exact identity
\[
 \overline{\mathcal R}_{N_{\mathrm{OOD}}}(b(\mathsf S))-\overline{\mathcal R}_{N_{\mathrm{OOD}}}(\mathsf S)
 =\overline{\mathcal R}_{N_{\mathrm{adapt}}}(b(\mathsf S))
   -\overline{\mathcal R}_{N_{\mathrm{adapt}}}(\mathsf S)-\Delta_{\mathsf S}.
\]
Substituting the source-risk bound proves Eq.~\ref{eq:adaptation}.

\end{proof}

\subsection{Controller adaptation and parameter accounting}
\label{app:freeze-proof}
For Fixed-Steer and LFS, selecting $C_\phi=\mathrm{id}$ and retaining
the pretrained readout recovers Fixed and Loop, respectively, at
their prescribed depths.
In particular, Proposition~\ref{prop:adapt} specializes to
\[
 \overline{\mathcal R}_{N_{\mathrm{OOD}}}(\mathrm{Loop})
   -\overline{\mathcal R}_{N_{\mathrm{OOD}}}(\mathrm{LFS})
 \ge \operatorname{Gap}_{\mathrm{LFS}}-2\gamma_{\mathrm{LFS}}-\Delta_{\mathrm{LFS}}.
\]
The same statement with Loop-Tune in place of LFS uses the same Loop
baseline but its own class, source gain, estimation term, and shift.
Neither statement compares LFS directly with Loop-Tune.
Sharing the controller architecture across schemes does not require
identical fitted weights or identical loss-class complexities.

For controller-only Fixed-Steer, $d_\phi$ counts all trainable controller
weights, including any learned size or layer-index embeddings.
Fixed-Tune updates all $d_E+d_R+d_D$ pretrained weights.
Loop-Tune also trains $d_E+d_R+d_D$ weights, while
LFS trains $d_\phi+d_D$.
Thus the fraction of Loop-Tune's trainable parameter budget used by LFS
is $(d_\phi+d_D)/(d_E+d_R+d_D)$. This is strictly below one when
$d_\phi<d_E+d_R$, and can be small when the processor dominates the budget.
For a common sample size and a common bounded-loss envelope,
the estimation terms in Lemma~\ref{lem:P3} have ratio
$\sqrt{p_{\mathrm{LFS}}/p_{\mathrm{Loop\text{-}Tune}}}$.
This ratio depends on loss-class complexity, not on the parameter
ratio alone; the prescribed depth can affect that complexity.

LFS may retain useful correction capacity with a small controller,
whereas Loop-Tune may represent corrections unavailable to that controller.
The theoretical benefit depends on $\operatorname{Gap}_{\mathsf S}$, $\gamma_{\mathsf S}$, and $\Delta_{\mathsf S}$ jointly.
The alignment error for a frozen processor is still the mismatch between the complete controlled update and $U^\star_{G_N}$.
Freezing does not imply that this mismatch decreases.

\subsection{When LFS can outperform Loop-Tune}\label{app:lfs-comparison}
Compare the two schemes at the same prescribed depth, conditional on the same pretrained model, with the same adaptation sample size.
Both retain the encoder and adapt the readout, so the comparison isolates processor fine-tuning versus controller adaptation.

\paragraph{An exact comparison.}
For this appendix only, define the expected adaptation-distribution
excess risk
\[
 e_{\mathsf S}:=
 \overline{\mathcal R}_{N_{\mathrm{adapt}}}(\mathsf S)
 -\inf_{g\in\mathcal H_{\mathsf S}}
       \risk_{N_{\mathrm{adapt}}}(g).
\]
Under Assumption~\ref{ass:extrapolation},
$0\le e_{\mathsf S}\le2\gamma_{\mathsf S}$.
The definitions of $\operatorname{Gap}_{\mathsf S}$ and
$\Delta_{\mathsf S}$ give the identity
\[
 \overline{\mathcal R}_{N_{\mathrm{OOD}}}(\mathsf S)
 =\overline{\mathcal R}_{N_{\mathrm{OOD}}}(\mathrm{Loop})
  -\operatorname{Gap}_{\mathsf S}+e_{\mathsf S}+\Delta_{\mathsf S}.
\]
Subtracting the identities for the two schemes yields
\begin{align}\label{eq:lfs-tune-comparison}
 &\overline{\mathcal R}_{N_{\mathrm{OOD}}}(\mathrm{Loop\text{-}Tune})
       -\overline{\mathcal R}_{N_{\mathrm{OOD}}}(\mathrm{LFS})\nonumber\\
 &\quad=
 \operatorname{Gap}_{\mathrm{LFS}}
       -\operatorname{Gap}_{\mathrm{Loop\text{-}Tune}}
 +(e_{\mathrm{Loop\text{-}Tune}}-e_{\mathrm{LFS}})
 +(\Delta_{\mathrm{Loop\text{-}Tune}}-\Delta_{\mathrm{LFS}}).
\end{align}
Thus LFS has smaller expected deployment risk precisely when its
estimation and transfer advantages together exceed any loss in
achievable adaptation gain. The classes need not be nested, so the
achievable-gain difference can have either sign.
This is an exact risk comparison, not a subtraction of upper bounds.
For example, a sufficient condition under the present assumptions is
\[
 \operatorname{Gap}_{\mathrm{LFS}}
 -\operatorname{Gap}_{\mathrm{Loop\text{-}Tune}}
 +\Delta_{\mathrm{Loop\text{-}Tune}}-\Delta_{\mathrm{LFS}}
 >2\gamma_{\mathrm{LFS}},
\]
using $e_{\mathrm{Loop\text{-}Tune}}\ge0$ and
$e_{\mathrm{LFS}}\le2\gamma_{\mathrm{LFS}}$.
A smaller upper bound $\gamma_{\mathrm{LFS}}$ alone does not prove
that the actual excess risk $e_{\mathrm{LFS}}$ is smaller.

\paragraph{From pretraining to adaptation.}
The ratio $N_{\mathrm{adapt}}/N_{\mathrm{train}}$ describes how far
adaptation moves beyond the pretraining size. If the pretrained local
update remains aligned across these sizes and the needed correction
can be expressed by the controller and readout, LFS can retain
comparable achievable gain while adapting fewer parameters.
If larger graphs expose a processor error that the controller cannot
correct, Loop-Tune can instead gain from updating the processor.
Consequently a large ratio can favor either scheme, depending on
whether the transferable processor structure or its inadequacy dominates.
Pretraining size enters the comparison through the pretrained weights
and resulting hypothesis classes, rather than an independent penalty
determined by the ratio alone.

\paragraph{From adaptation to deployment.}
The ratio $N_{\mathrm{OOD}}/N_{\mathrm{adapt}}$ describes the remaining
size change after fitting. If processor fine-tuning learns corrections
specific to the adaptation distribution while LFS preserves useful
cross-size dynamics, LFS can have the smaller relative shift
$\Delta_{\mathsf S}$. Conversely, transferable processor corrections
can favor Loop-Tune; a size-conditioned controller also requires
control outside its adaptation sizes.
The local-error transfer and correction bounds in
App.~\ref{app:shift} make these requirements explicit.
The present assumptions do not imply that this shift, or the difference
between schemes, grows monotonically with the size ratio.
Different-depth deployment additionally requires the joint control
in App.~\ref{app:depth-schedules}.

\paragraph{Data budget and absolute sizes.}
The estimation term depends on the number of independent adaptation
graphs $m_{\mathrm{adapt}}$, not their node count $N_{\mathrm{adapt}}$.
With comparable achievable gains, limited graph samples make a
lower-complexity adaptation class attractive; abundant samples can
make additional corrective capacity more valuable.
Neither the two size ratios nor their product
$N_{\mathrm{OOD}}/N_{\mathrm{train}}$ determines the risk ordering:
absolute sizes, attributed neighborhood laws, target dynamics,
depth, and sample budget also matter.
The favorable regime for LFS is therefore a reusable pretrained
processor, corrections expressible by a small controller and readout,
and sufficiently persistent adaptation benefits.

\subsection{Approximate empirical minimization}
If the returned predictor has empirical risk within $\eta_{\mathsf S}$
of the class infimum, add $\eta_{\mathsf S}$ to the right-hand side
of Eq.~\ref{eq:erm}, add $4\eta_{\mathsf S}$ to the explicit deployment
upper bound, and subtract $\eta_{\mathsf S}$ in the adaptation-gain
lower bounds. Regularized fitting requires accounting for its penalty
or interpreting the regularizer as a restriction of the fitted class.

\section{When logarithmic resolution depth follows}\label{app:geometric}
The main analysis does not assume a geometric resolution profile. To
retain the manuscript's more specific logarithmic law, it must be
established separately for the target--distribution family in question.

\begin{assumption}[Optional geometric resolution profile]\label{ass:resolution}
There exist $\beta>0$ and $0<c_-\le c_+<\infty$, independent of $N,k$,
such that, for integer $0\le k\le T_N^\star$,
\begin{equation}\label{eqn:riskassumption}
 c_-N^{2\beta}q_N(k)\le\err_N(k)\le a_N(k)
            \le c_+N^{2\beta}q_N(k),
 \qquad
 q_N(k)=\frac{\rho_N^{2k}-\rho_N^{2T_N^\star}}
                     {1-\rho_N^{2T_N^\star}}.
\end{equation}
Here the task horizon is a finite integer $T_N^\star\ge1$.
An equilibrium variant with an explicit depth window is given below.
\end{assumption}

This assumption imposes both non-negligible information outside local
neighborhoods and near-optimality, up to constants, of truncated task
computation. Its lower direction does not follow from contractivity.
At $k=0$ it requires initial local uncertainty of order $N^{2\beta}$.

\begin{corollary}[Two-sided logarithmic resolution depth]\label{cor:log-depth}
Under Assumption~\ref{ass:resolution}, suppose
$0<\underline\rho\le\rho_N\le\bar\rho<1$ and
$\rho_N^{2T_N^\star}=O(N^{-2\beta})$. For fixed $\epsilon>0$,
the depths defined in the main text satisfy
\begin{equation}\label{eq:log-depth}
 L_\epsilon(N)=\frac{\beta\log N}{|\log\rho_N|}+O_\epsilon(1),
 \qquad
 K_{\mathrm{res}}(N,\epsilon)
       =\frac{\beta\log N}{|\log\rho_N|}+O_\epsilon(1).
\end{equation}
Both depths are $\Theta(\log N)$. For fixed $K_{\mathrm{fix}}$,
fixed-radius risk is $\Omega(N^{2\beta})$. If Loop's encoder,
per-step update, and decoder errors and its learned decoder Lipschitz
constant are uniformly bounded at $K_{\mathrm{res}}$, Loop risk
is $O(1)$ and eventually strictly smaller. If also $\log N=o(D_N)$,
this specialization belongs to the growing sub-diameter regime.
\end{corollary}
\begin{proof}
For $z_N=\rho_N^{2T_N^\star}$ and any fixed $x>0$,
\[
 q_N(k)\le xN^{-2\beta}
 \quad\Longleftrightarrow\quad
 \rho_N^{2k}\le z_N+(1-z_N)xN^{-2\beta}.
\]
The right-hand side is bounded above and below by positive constant
multiples of $N^{-2\beta}$, since $z_N=O(N^{-2\beta})$.
Taking logarithms and using the uniform bounds on $\rho_N$ locates
this threshold at $\beta\log N/|\log\rho_N|+O(1)$. The two sides of
\eqref{eqn:riskassumption}, with $x=\epsilon^2/c_-$ and
$x=\epsilon^2/c_+$, bracket both minima in Eq.~\ref{eq:log-depth}.
Also $z_N\to0$ and
$\rho_N^{2K_{\mathrm{fix}}}\ge\underline\rho^{2K_{\mathrm{fix}}}$,
so the lower profile yields the polynomial ceiling. Uniform contraction
away from one and the stated error bounds make the first term of
Eq.~\ref{eq:alignment-risk} bounded; its second term is at most
$2\epsilon^2$. Hence Loop risk is $O(1)$.
\end{proof}

\paragraph{Scope of the special case.}
The profile controls both information and the accuracy of truncated
computation. Contractivity alone supplies an upper envelope. With bounded
$A_N$ and contraction away from one, the general proposition instead
gives bounded depth for fixed tolerance. Growing depth requires
appropriate scale or contraction dependence on size, or a changing
tolerance. The formula concerns tolerance-reaching depth, not the
empirical risk-minimizing depth or its fitted slope across arbitrary
horizon schedules. For finite horizons $O_\epsilon(1)$ is stated for
fixed $\epsilon$, not as a uniform small-tolerance asymptotic.

\paragraph{Equilibrium targets and the depth window.}
For an equilibrium target, assume instead that
\[
 c_-N^{2\beta}\rho_N^{2k}\le\err_N(k)\le a_N(k)
       \le c_+N^{2\beta}\rho_N^{2k},\qquad 0\le k\le J_N,
\]
where $J_N$ is an integer depth limit satisfying
$N^{2\beta}\rho_N^{2J_N}\to0$. Under the same uniform bounds on
$\rho_N$, the resolution thresholds lie inside this window for every
fixed $\epsilon>0$, and
\[
 L_\epsilon(N),\ K_{\mathrm{res}}(N,\epsilon)
 =\frac{\beta\log N+\log(1/\epsilon)}{|\log\rho_N|}+O(1).
\]
Here the notation means each depth equals the displayed expression up
to a bounded additive term. This follows by solving the two geometric
threshold inequalities. A depth window is essential: on a connected
$N$-node graph, radius $N-1$ reveals the entire attributed graph, so
$\err_N(N-1)=0$. A strictly positive geometric lower profile cannot
hold at all finite radii for an equilibrium target. The finite-horizon
profile above must likewise be checked for compatibility with the
available graph information; it is a specialization rather than a
general consequence of iteration.

\section{Depth selection and depth schedules}\label{app:depth-schedules}
\paragraph{Selecting depth for implemented-model accuracy.}
The depth $K_{\mathrm{res}}$ concerns the task computation. To select
using the implemented model's error bound, let $\mathcal K$ be a
specified finite set of candidate depths and suppose
Lemma~\ref{lem:accum} supplies valid bounds at every candidate. If the
following set is nonempty, define
\[
 K_{\mathrm{cert}}(N,\epsilon)
 :=\min\{K\in\mathcal K:\mathcal U_{\mathrm{Loop},N}(K)\le\epsilon^2\}.
\]
Then $\risk_N(f_{N,K_{\mathrm{cert}}}^{\mathrm{Loop}})\le\epsilon^2$.
If computation increases with depth, this is the least costly candidate
certified by the bound to attain the chosen accuracy. It need not
minimize the true risk, nor rule out success at a smaller uncertified
depth. The error terms must be known or validly bounded to implement
this rule. If alignment constants worsen with depth, that dependence
must be included in $\mathcal U_{\mathrm{Loop},N}(K)$.

\paragraph{Different source and deployment depths.}
If adaptation fits at $K_n$ and deployment runs at $K_N$, prescribe the
schedule before adaptation and replace the same-depth transfer condition
with the joint condition
\[
 \sup_{g\in\mathcal H_{\mathsf S}}
 |\calA_{N,K_N}(g)-\calA_{n,K_n}(g)|
 \le\tau_{\mathsf S}(N,n;K_N,K_n).
\]
The deployment lemma then has lower bound $\err_N(K_N)$ and computation
errors $a_n(K_n)$ and $a_N(K_N)$. All source risks and the estimation
class must use the depth actually fitted. The Loop-Tune/LFS comparisons extend
when the baseline and adapted model share the prescribed schedule.
Same-depth transfer does not by itself establish the joint condition.
If depth is selected using adaptation data, its selection must be
included in the hypothesis class or separately analyzed. Deployment
labels used for selection require a separate evaluation sample.

\section{Controlling the relative deployment shift}\label{app:shift}
Proposition~\ref{prop:adapt} characterizes adaptation improvement in terms of
$\Delta_{\mathsf S}$. To obtain a sufficient condition not expressed
through an unknown deployment-risk difference, the following result
connects it to local-risk transfer and the size of the fitted correction.

\begin{corollary}[Sufficient adaptation certificate]\label{cor:adapt-certificate}
Suppose Assumption~\ref{ass:size-transfer} holds for
$\mathcal H_{\mathsf S}$ containing its baseline $f_{b(\mathsf S)}$, and finite
constants $c_{\mathsf S}(j)$ satisfy
\[
 \sup_{g\in\mathcal H_{\mathsf S}}
 \left(\E\|g_j(G_j,v)-f_{b(\mathsf S),j}(G_j,v)\|^2\right)^{1/2}
\le c_{\mathsf S}(j),\qquad
j\in\{N_{\mathrm{adapt}},N_{\mathrm{OOD}}\}.
\]
Then
\begin{align}\label{eq:shift-certificate}
 |\Delta_{\mathsf S}|
 &\le2\tau_{\mathsf S}(N_{\mathrm{OOD}},N_{\mathrm{adapt}},K)
     +2c_{\mathsf S}(N_{\mathrm{OOD}})
        \sqrt{a_{N_{\mathrm{OOD}}}(K)}\nonumber\\
 &\quad+2c_{\mathsf S}(N_{\mathrm{adapt}})
        \sqrt{a_{N_{\mathrm{adapt}}}(K)}.
\end{align}
Any of the four adaptation schemes improves on its own baseline when
$G_{\mathsf S}$ exceeds $2\gamma_{\mathsf S}$ plus the right-hand side
of Eq.~\ref{eq:shift-certificate}.
\end{corollary}
\begin{proof}
For a fixed predictor, expand the risk difference from the baseline:
\[
 \risk_j(g)-\risk_j(f_{b(\mathsf S)})
 =\calA_{j,K}(g)-\calA_{j,K}(f_{b(\mathsf S)})
 +2\E\langle g_j-f_{b(\mathsf S),j},\widetilde Y_{j,K}-Y_j\rangle.
\]
The pure task-computation errors cancel. Compare
$j=N_{\mathrm{OOD}}$ and $j=N_{\mathrm{adapt}}$,
apply the transfer bound to both predictors, and bound the cross terms
by Cauchy--Schwarz. Uniformity permits averaging over the fitted
predictor and gives Eq.~\ref{eq:shift-certificate}. Substitute in the
appropriate adaptation-gain bound.
\end{proof}

\section{Justifying the transfer assumption}\label{app:transfer}
Let $Z_{j,K}=B_K(G_j,v)$ have law $\mu_{j,K}$ on a common metric space
with finite first moments. Suppose
$\calA_{j,K}(g)=\int\psi_{j,g,K}(z)\,d\mu_{j,K}(z)$.
If, uniformly over $g\in\mathcal H_{\mathsf S}$, the function
$\psi_{N,g,K}$ is $\Lambda_{\mathsf S,K}$-Lipschitz and
\[
 \sup_z|\psi_{N,g,K}(z)-\psi_{n,g,K}(z)|
       \le\chi_{\mathsf S}(N,n,K),
\]
then
\begin{equation}\label{eq:transfer-sufficient}
 |\calA_{N,K}(g)-\calA_{n,K}(g)|
 \le\Lambda_{\mathsf S,K}W_1(\mu_{N,K},\mu_{n,K})
       +\chi_{\mathsf S}(N,n,K).
\end{equation}
Indeed, insert $\int\psi_{N,g,K}\,d\mu_{n,K}$, apply the Lipschitz
coupling bound to the first difference, and bound the second difference
by $\chi_{\mathsf S}$. An integrated bound under $\mu_{n,K}$ suffices in place
of the displayed supremum.

For a common size-independent predictor and task rule, $\chi_{\mathsf S}=0$.
For a size-conditioned controller, it need not vanish. Even if the
neighborhood laws coincide, a predictor can change its output with
$N$, so graph-distribution convergence alone cannot bound its risk shift.
The Lipschitz constant of a squared-loss map also needs control of
prediction/target magnitudes (or an alternative weighted argument).

\paragraph{A separate way to control the adaptation certificate.}
The correction bounds $c_{\mathsf S}(j)$ in Corollary~\ref{cor:adapt-certificate}
can be enforced by a restricted class, but are not implied by freezing.
For example, for $\mathsf S\in\{\text{Fixed-Steer},\mathrm{LFS}\}$, restrict to a controller-only
class using the shared processor $R_{\theta_0,G}$ and controller
$C_\phi(\,\cdot\,,\mathrm{ctx}^{(t)})$, with the same encoder and decoder
as the baseline. Suppose the processor is $\lambda$-Lipschitz,
the decoder is $L_D$-Lipschitz, and
\[
 \left(\E_G\left\|
 C_\phi(h_t^{\mathsf S},\mathrm{ctx}^{(t)})
       -h_t^{\mathsf S}\right\|_j^2\right)^{1/2}
 \le\delta,\qquad 0\le t<K,
\]
uniformly over the class at graph size $j$.
Comparing controlled and uncontrolled states gives
\[
 c_{\mathsf S}(j)\le L_D\lambda\delta
                   \sum_{t=0}^{K-1}\lambda^t.
\]
This bound is uniform in depth if $\lambda<1$, grows linearly for
$\lambda=1$, and can grow exponentially for $\lambda>1$.
It controls departure from the baseline, not alignment with the task
operator or an automatic improvement in accuracy.

\section{Discussion of task structure and scope}\label{app:discussion}
\subsection{Homogeneity and dependency radius}\label{app:homogeneity}
Dependency radius measures the information needed for a prediction,
whereas homogeneity describes its response to rescaling input
attributes. For a fixed graph topology $G$ and specified numerical
attributes $x$, positive homogeneity of degree $p$ means
\[
 f_N^\star(G,\lambda x)=\lambda^p f_N^\star(G,x),
 \qquad \lambda>0.
\]
Only the specified attributes are rescaled; topology and discrete
markers are held fixed. Shortest-path distances and optimal TSP
\emph{costs} scale linearly with edge weights; the optimal paths or
tours themselves need not change. Similarly, maximum-flow values
scale linearly when all edge capacities are rescaled. These numerical
scaling properties motivate homogeneous graph architectures such as
IterGNN~\cite{tang2020itergnn}; graph-level optimization values here
illustrate homogeneity rather than node-level regime membership.

Homogeneity alone determines neither the dependency-radius regime
nor the existence of a contractive iterative representation. The
Gaussian local and global sums above give degree-one homogeneous
examples in the three regimes under their stated conditions. Adding
a known nonzero constant gives non-homogeneous examples with exactly
the same optimal local approximation errors. Nevertheless, input
rescaling changes absolute squared error and hence can change
$L_\epsilon(N)$ at a fixed tolerance. Radius and homogeneity therefore
describe distinct properties, with radius always interpreted relative
to the specified distribution and accuracy criterion.

\subsection{Iterative tasks and the scope of contractivity}
\label{app:iterative-tasks}
The iterative representation specifies a common local update reused
across steps and graph sizes. Assumption~\ref{ass:contraction} adds
strict contractivity for the error-accumulation analysis; it is not
implied by that representation or by a dependency-radius regime.
The following examples separate these requirements.

\paragraph{Restart diffusion and label propagation.}
A concrete template is
\[
 U_G^\star(h)=\alpha P_Gh+(1-\alpha)s_G,
 \qquad 0<\alpha<1,
\]
where $s_G$ is a fixed node-wise input and $P_G$ propagates information
over one hop. If $\|P_G\|_{N\to N}\le1$, then this update is
$\alpha$-contractive in the per-node norm used in our analysis.
Symmetric normalized adjacency on an undirected graph satisfies this
operator-norm condition. Restart label-score propagation
\cite{zhou2003learning} and the propagation stage of
APPNP~\cite{klicpera2019appnp} have this affine structure.
APPNP is an architectural example, rather than a separate task target;
its initial neural predictions play the role of the fixed source.
Here the source is conditioned on as an input attribute, not included
as an independently varying hidden-state component in the contraction
inequality. Locality also requires any normalization information to
be available as local input attributes.

Classical PageRank uses the same restart principle, but its standard
stochastic-matrix contraction must be interpreted in the appropriate
norm: column-stochastic propagation is non-expansive in $\ell_1$,
whereas row-stochastic propagation is non-expansive in $\ell_\infty$.
Neither property alone establishes contraction in our per-node
Euclidean norm. Moreover, fixed damping and bounded signal scale can
give bounded sufficient depth at fixed tolerance, as shown above.
A logarithmic size dependence or membership in
$\mathcal F_2^{(\epsilon)}$ requires additional conditions; it does
not follow merely from using a restart iteration. Equilibrium outputs
use an infinite horizon, while finite-horizon targets specify the
number of updates explicitly.

\paragraph{Diffusion and application-specific dynamics.}
Graph diffusion motivates iterative architectures such as
GRAND~\cite{chamberlain2021grand}, but ordinary heat diffusion is not
strictly contractive on the full state space. For a symmetric graph
Laplacian $\mathcal L_G$, the Euler update $I-\eta\mathcal L_G$
has eigenvalue one on the Laplacian's nullspace. A leaky version
\[
 U_G^\star(h)=\bigl[I-\eta(\mathcal L_G+\mu I)\bigr]h+\eta s_G,
 \qquad \mu>0,
\]
is strictly contractive when
$0<\eta<2/(\lambda_{\max}(\mathcal L_G)+\mu)$.
A size-independent step size needs a uniform spectral bound.
Belief propagation, routing updates, and power-flow solvers require
their own update, state-space, locality, and contraction checks.
Damping alone does not guarantee strict contraction, and a local
convergence statement near an operating point does not establish the
global condition in Assumption~\ref{ass:contraction}.

\paragraph{Iterative computation beyond the contractive analysis.}
Neighborhood maxima and Bellman--Ford distance updates illustrate
local repeated computation whose usual update is non-expansive,
rather than strictly contractive, in the sup norm on finite states.
For example, the unanchored Bellman update satisfies
$U(h+c\mathbf 1)=U(h)+c\mathbf 1$, preserving constant differences.
Such examples lie beyond our stated contraction-based guarantee;
this does not rule out alternative representations or analyses.
Their dependency-radius regime must still be established from
the target distribution and tolerance. Across all examples, iterative
structure motivates processor reuse, contractivity controls state-error
accumulation, and local approximation error determines information reach.


\section{Additional Experimental Results}
\label{app:experiments-full}

This appendix contains the full experimental study summarized in
Sec.~\ref{sec:experiments}. We first give the full experimental protocol
(App.~\ref{app:exp-protocol}) and verify the two main assumptions
without training any model (App.~\ref{app:exp-primitives}). We then
report the complete risk curves of the four schemes across the synthetic
task families (App.~\ref{app:exp-tiers}), results on a natural
fixed-point task (App.~\ref{app:exp-anchors}), the effect of the
adaptation budget and of the size gap (App.~\ref{app:exp-fsft}),
direct measurements of how adaptation changes the learned processor
(App.~\ref{app:exp-a14}), and the behavior of all schemes outside
the contractive class (App.~\ref{app:exp-scope}). The tasks,
protocol, and notation are those of Sec.~\ref{sec:exp-synthetic} unless
stated otherwise. Code, seeds, and raw results are released, and all
datasets can be regenerate deterministically from seeds or from existing open-source benchmark.

\subsection{Protocol details}
\label{app:exp-protocol}

All models are pretrained on graphs with $100$--$200$ nodes using depths
$K \in [2, 8]$, adapted on $100$ labeled graphs with $350$--$700$ nodes,
and evaluated on held-out graphs with $10^3$ to $2 \times 10^4$ nodes.
At the largest size and $d = 2$, the horizons of the four families are
$5$ (\textsc{const}), $9$ (\textsc{log}), $16$ (\textsc{poly}), and
$50$ (\textsc{diam}) steps. Loop-Tune fine-tunes all $265{,}857$
parameters of the looped model, including the encoder; LFS trains the
controller and the readout, $42{,}817$ parameters or $16\%$ of the
model. The Fixed model is implemented in the standard way, with a
distinct set of weights per layer, so it has $1.06$M parameters against
the looped model's $265{,}857$. The Fixed-versus-Loop comparison
therefore differs in weight sharing as well as in depth, with the fixed
model the larger of the two: it compares the loop against the
fixed-depth architecture a practitioner would train, not against the
same weights deployed at a fixed depth. The $4\%$ adaptation fraction
of Fixed-Steer is relative to this larger model. For the looped schemes, the number of iterations used at test
time is chosen from data: at each test size we set aside two validation
graphs, run each scheme on them at a range of depths, and keep the depth
with the lowest validation error; the scheme is then evaluated at that
depth on eight further test graphs of the same size.

\textbf{Controller.} The controller $C_\phi$ is a feature-wise affine
modulation (FiLM) of the hidden state, applied per node between
applications of the frozen processor. At iteration $k$ on a graph with
$N$ nodes, a conditioning vector
$z = \mathrm{MLP}\big([\log N / \log 10^4,\; k/K_{\max},\;
\sin(2\pi k/K_{\max}),\; \cos(2\pi k/K_{\max})]\big)$
is computed by a two-layer MLP of width $32$ with SiLU activations
($K_{\max} = 256$), and a linear head maps $z$ to per-channel scales and
shifts $(\gamma, \beta)$. The controller output is
$C_\phi(h, \mathrm{ctx}) = h \odot (1 + \gamma) + \beta$. The head is
zero-initialized, so $C_\phi$ is the identity at $\phi = 0$ and
adaptation starts from the unadapted looped model. This gives
$d_\phi = 9{,}664$ parameters at hidden width $128$. The chip-placement
network uses a FiLM controller of the same form between its processor
applications. Fixed-Steer inserts the same module between the layers of
the fixed-depth model, indexed by layer.

\textbf{Chip placement details.} The pretrained placement network
applies its three-block processor once; the looped schemes re-apply it
as a unit, and we deploy them at $K = 2$ unrolls (deeper unrolls behave
consistently and are omitted). Pretraining and adaptation both minimize
the mean squared error to reference placements (target supervision; the
intermediate states of a reference placer are not available, so there is
no operator-aligned variant on this task). Each synthetic size bucket
contains $20$ held-out designs; the MLCAD column evaluates $1{,}000$
held-out samples from netlists not seen during adaptation, under a
design-level split. Every adapted cell is the mean of $3$ independent
adaptation runs from the same pretrained network. Fixed-Tune and
Loop-Tune adapt all $5{,}851{,}650$ parameters; LFS adapts $169{,}858$
(the $37{,}760$-parameter controller plus the $132{,}098$-parameter
readout), $2.9\%$ of the model, with correspondingly smaller optimizer
state and adaptation checkpoints ($34\times$). The first two size
buckets overlap the pretraining ($\le 500$ cells) and adaptation
($500$--$1{,}000$ cells) ranges; only the $1800$--$2000$ and
$2400$--$2800$ buckets are size extrapolation.

\textbf{Computing infrastructure.} All experiments ran on a single
internal lab server with NVIDIA RTX A6000 GPUs (48\,GB), AMD EPYC 7452
CPUs, and 1\,TB of RAM. Every run fits on one GPU, and at most two GPUs
were used concurrently for the chip placement experiments. The full synthetic study (all cells, 5 seeds
each) takes roughly 40 GPU-hours; each chip-placement adaptation run
takes 2--13 GPU-hours depending on the loop depth. Implementations
build on PyTorch~\cite{paszke2019pytorch} (BSD-3-Clause) and PyTorch
Geometric~\cite{fey2019pyg} (MIT). The real-netlist experiments use the
2023 MLCAD FPGA macro-placement contest
benchmark~\cite{bustany2023mlcad}, distributed under the MIT license;
all other datasets are generated by our released code.

\subsection{Verifying the analysis assumptions without training}
\label{app:exp-primitives}

\begin{figure}[t]
  \centering
  \includegraphics[width=0.92\textwidth]{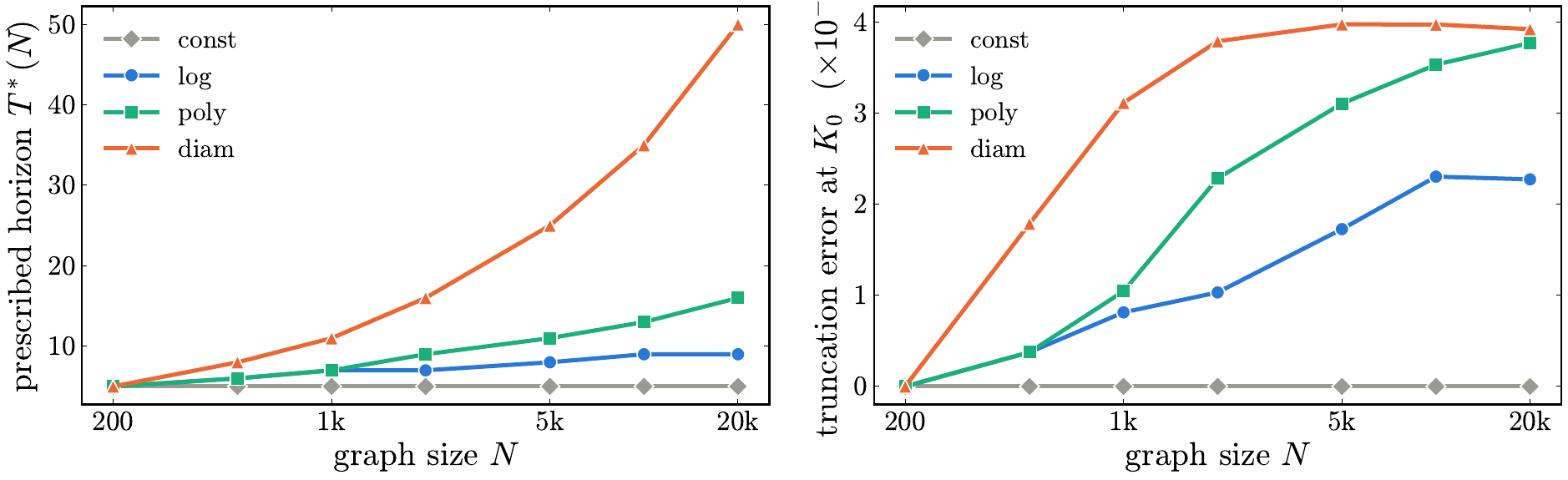}
  \caption{\textbf{The four task families and their fixed-radius error.}
  Computed from the task definition, before any learning, with RGG $d{=}2$,
  $\rho{=}0.7$. Left: the prescribed horizon
  $T^\star(N)$ of each family. Right: the per-node error of the
  $K_0$-step truncation, which uses only a $K_0$-hop neighborhood. For
  \textsc{const} this error is exactly zero at every size. For the other
  families it grows with the graph size. The truncation error
  upper-bounds the fixed-radius floor of
  Proposition~\ref{prop:ceiling}.}
  \label{fig:signature}
\end{figure}

Fig.~\ref{fig:signature} shows the training-free check on the
constructed families: the prescribed horizons $T^\star(N)$ and the
per-node truncation error at the pretraining radius $K_0$. The
truncation error decreases geometrically with the radius $k$, consistent
with the $\rho^{2k}$ decay implied by the contractive update
(Assumption~\ref{ass:contraction}). A predictor using only a $K_0$-hop
neighborhood can achieve lower error than the truncation only when
nearby inputs contain information about more distant ones, and our
smooth input field partially allows this (App.~\ref{app:exp-tiers}).

\begin{figure}[t]
  \centering
  \begin{subfigure}[t]{0.32\textwidth}
    \includegraphics[width=\linewidth]{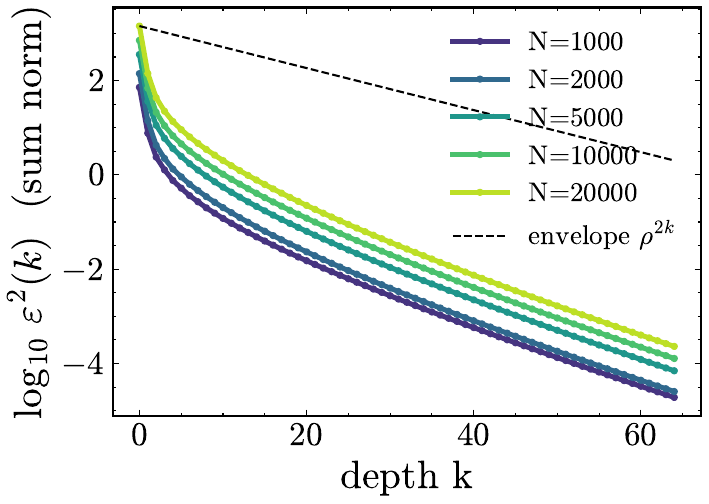}
    \caption{Error of radius-$k$ prediction (Assumption~\ref{ass:contraction}).}
    \label{fig:resprofile}
  \end{subfigure}\hfill
  \begin{subfigure}[t]{0.32\textwidth}
    \includegraphics[width=\linewidth]{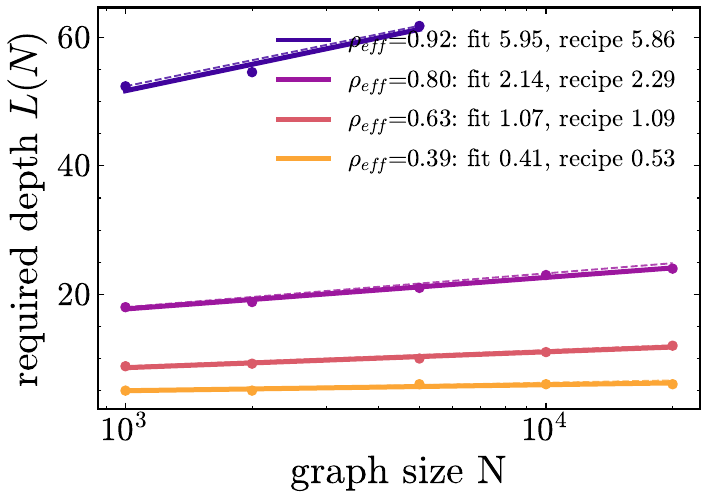}
    \caption{Required depth vs.\ graph size.}
    \label{fig:depthlaw}
  \end{subfigure}\hfill
  \begin{subfigure}[t]{0.32\textwidth}
    \includegraphics[width=\linewidth]{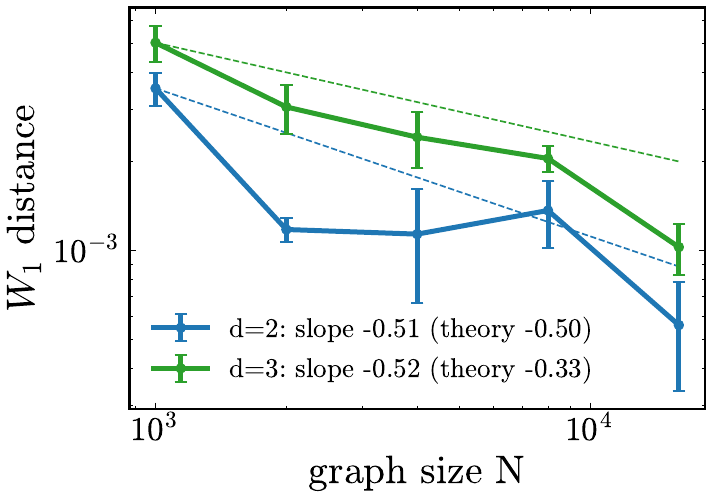}
    \caption{Cross-size transfer (Assumption~\ref{ass:size-transfer}).}
    \label{fig:w1}
  \end{subfigure}
  \caption{\textbf{Training-free checks of the two main assumptions}, on
  damped PageRank (Appendix~\ref{app:exp-anchors}).
  (a) The error of the best radius-$k$ prediction decreases geometrically
  in $k$; the guaranteed rate $\rho^{2k}$ (dashed) is an upper bound, and
  the measured rate $\rho_{\mathrm{eff}}$ is somewhat faster.
  (b) The radius needed to reach a fixed accuracy grows linearly in
  $\log N$; using the \emph{measured} rate $\rho_{\mathrm{eff}}$ predicts
  the growth rate within about $10\%$, whereas using the nominal $\rho$
  overestimates it by up to $10\times$ on expander graphs.
  (c) The distance between local neighborhood distributions at size $N$
  and at a large reference size decays as $N^{-1/d}$: the measured
  exponent is $-0.510$ against the predicted $-0.500$ at $d=2$
  (16 seeds); at $d=3$ the measurement ($-0.52$) is consistent with the
  bound but the graphs are too small to reach the asymptotic rate.}
  \label{fig:primitives}
\end{figure}

In this subsection we verified two assumptions that carry the risk decomposition: the geometric
decay of the radius-$k$ prediction error, implied by the contractive
target (Assumption~\ref{ass:contraction}), and the $N^{-1/d}$ transfer rate
between sizes (Assumption~\ref{ass:size-transfer}). 
Fig.~\ref{fig:primitives} shows the results on damped PageRank, where
the contraction rate is known.
We show that the guaranteed decay rate $\rho^{2k}$ is a valid but conservative
bound and the measured decay is faster. Therefore any practical rule for choosing
the deployment depth should use the measured rate. On sparse
Erd\H{o}s--R\'enyi graphs the measured rate is well approximated by
$\rho\,\lambda_2(P)$, where $\lambda_2$ is the second eigenvalue of the
transition matrix. The spectral gap of the graph and 
damping together determines how quickly information is resolved.

\subsection{Risk of the four schemes across task families}
\label{app:exp-tiers}

\begin{figure}[t]
  \centering
  \includegraphics[width=\textwidth]{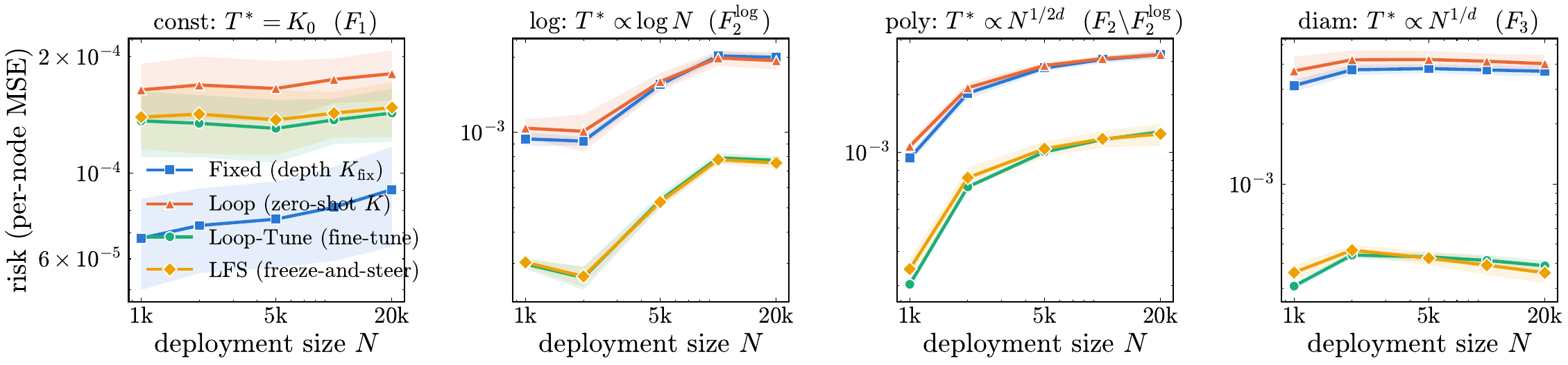}
  \caption{\textbf{Test risk of the four schemes across deployment
  sizes}, one panel per task family (target-supervised pretraining, RGG
  $d{=}2$, $\rho{=}0.7$; 5 seeds, $\pm1$~sd; each scheme selects its
  deployment depth on a held-out split at every size). On the
  $\mathcal{F}_1$ family all four schemes perform alike and no scheme
  degrades with size. On the other families, the risk of the two
  unadapted schemes grows with the graph size while the two adapted
  schemes are $2.6$--$10\times$ more accurate, and LFS
  gradually overtakes Loop-Tune as the size grows, most clearly on
  \textsc{diam}.}
  \label{fig:hier-tiers}
\end{figure}

Fig.~\ref{fig:hier-tiers} shows the complete risk curves for the
summary in Sec.~\ref{sec:exp-synthetic}. We observed that on the
\textsc{const} family all four schemes have risk curves within $5 \times 10^{-5}$
and none degrades with size, as expected when a fixed radius is sufficient.
On every other family the unadapted schemes degrade as the graph size
grows, in line with the error floor of
Proposition~\ref{prop:ceiling}, while adaptation at intermediate sizes
removes most of the degradation. Under end-to-end target-supervised
pretraining the looped model does \emph{not} improve over the fixed-depth
model, for the
reason established in Sec.~\ref{sec:exp-synthetic}. Under operator-aligned pretraining the ordering
changes: the looped model outperforms the fixed-depth model in $4$--$5$ of $5$
seeds in every $\mathcal{F}_2/\mathcal{F}_3$ setting (for example
$0.0144$ versus $0.0459$ on \textsc{diam} with $\rho = 0.9$ at
$N = 2\times10^4$), and Loop-Tune improves another $2$--$6\times$.

\begin{figure}[t]
  \centering
  \includegraphics[width=0.62\textwidth]{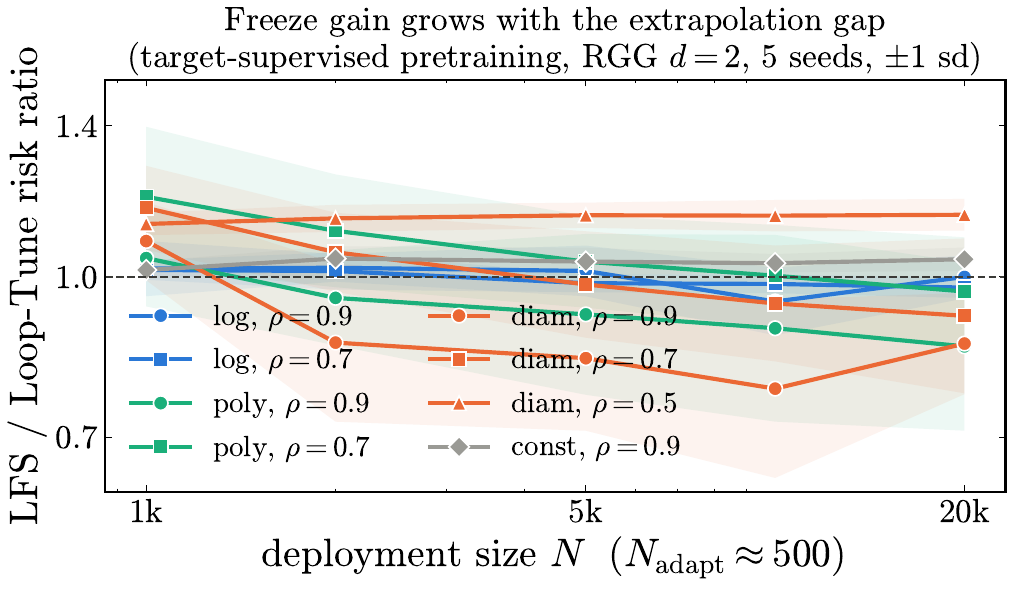}
  \caption{\textbf{Risk ratio of LFS to Loop-Tune as a
  function of deployment size} (target-supervised pretraining, RGG
  $d{=}2$; 5 seeds, $\pm1$~sd; adaptation on 100 graphs of $350$--$700$
  nodes). Values below $1$ mean LFS is more accurate. Near
  the adaptation sizes Loop-Tune leads; the curves cross for
  \textsc{poly} and \textsc{diam} between $2$ and $20$ times the
  adaptation size at $\rho \ge 0.7$, \textsc{log} stays within a few
  percent of parity, and where a small neighborhood suffices
  ($\rho{=}0.5$, and \textsc{const} at every $\rho$) Loop-Tune leads
  at every size.}
  \label{fig:hier-fsft}
\end{figure}

\textbf{Fixed-depth adaptation variants.}
Tab.~\ref{tab:tier0-variants} adds the two fixed-depth adaptation
variants of Sec.~\ref{sec:schemes}: Fixed-Tune fine-tunes all parameters of
the fixed-depth model on $D_{\mathrm{adapt}}$, and Fixed-Steer trains
the between-layer adapter and the readout ($42{,}817$ parameters,
$4.0\%$ of the fixed model's $1.06$M). Two observations. Each fixed-depth
variant performs on par with its looped counterpart (Fixed-Tune with
Loop-Tune, and Fixed-Steer with LFS): at this contraction rate the
optimal deployment depth stays near the pretraining depth
(Sec.~\ref{sec:exp-synthetic}), so once the model is adapted, the loop
itself adds little under target-supervised pretraining. Fixed-Steer
matches or improves on Fixed-Tune on every family while training $4\%$
of the parameters. The Fixed rows of this
table come from the same runs as the variants; they differ from the
corresponding runs of Fig.~\ref{fig:hier-tiers} only within seed
noise because the two studies pretrain independent models under the same
protocol. One further observation: the adapted fixed-depth variants
reach errors below the truncation values of Fig.~\ref{fig:signature}.
This is possible because our smooth input field has few modes, so its
distant values are partially predictable from a local patch, and
adaptation teaches the fixed-depth models this inference. The truncation
error is therefore an upper bound on what a fixed-radius predictor can
achieve on these inputs, and the gap between the two closes as the input
field carries more local randomness.

\begin{table}[t]
  \centering
  \caption{\textbf{All six schemes of Sec.~\ref{sec:schemes}} on the
  constructed families (RGG $d{=}2$, $\rho{=}0.7$; per-node MSE
  $\times 10^{-3}$, mean over 5 seeds, sd $\le 0.3$ in these units for
  every entry; deployment at $N = 2\times10^4$, i.e.\ $40\times$ the
  adaptation sizes). Best per column in bold.}
  \label{tab:tier0-variants}
  \small
  \begin{tabular}{llcccc}
    \toprule
    scheme & depth / adapted params & \textsc{const} & \textsc{log} & \textsc{poly} & \textsc{diam} \\
    \midrule
    Fixed       & fixed / none        & 0.10 & 2.00 & 3.31 & 3.71 \\
    Fixed-Tune  & fixed / all         & 0.09 & 0.77 & 1.36 & 0.38 \\
    Fixed-Steer & fixed / adapter + readout & \textbf{0.07} & \textbf{0.74} & \textbf{1.06} & \textbf{0.25} \\
    Loop        & tunable / none      & 0.18 & 1.94 & 3.26 & 4.01 \\
    Loop-Tune   & tunable / all       & 0.14 & 0.77 & 1.28 & 0.39 \\
    LFS         & tunable / controller + readout & 0.15 & 0.75 & 1.25 & 0.36 \\
    \bottomrule
  \end{tabular}
\end{table}

The picture changes when the computation genuinely needs more depth.
Tab.~\ref{tab:tier0-variants-op} repeats the comparison at slower
contraction ($\rho = 0.9$) with operator-aligned pretraining, on the two
families whose dependency radius grows fastest; the constructed radius
at the largest size is about $10$ for \textsc{log} and $50$ for
\textsc{diam}, well beyond the fixed depth $K_{\mathrm{fix}} = 5$. Here
the ranking separates along both axes of Sec.~\ref{sec:schemes}.
Adaptation at fixed depth recovers only part of the error: Fixed-Tune
removes $16$--$18\%$ and Fixed-Steer $24$--$38\%$ of the unadapted
error. The loop
alone, with no adaptation at all, already beats every fixed-depth
variant. Adapting the looped model then cuts the error further, to
$3$--$11\times$ below the best fixed-depth variant, with the adapted
looped models selecting deployment depths of $8$--$9$ (\textsc{log}) and
$16$--$24$ (\textsc{diam}). Together the two tables give the summary:
when the task's dependency radius stays near the pretraining depth,
adaptation matters and the loop adds little
(Table~\ref{tab:tier0-variants}); when the radius grows past it, the
loop is necessary and adaptation alone cannot substitute for it
(Table~\ref{tab:tier0-variants-op}).

\begin{table}[t]
  \centering
  \caption{\textbf{The six schemes at slow contraction under
  operator-aligned pretraining} (RGG $d{=}2$, $\rho{=}0.9$; per-node MSE
  $\times 10^{-3}$, mean $\pm$ sd over 5 seeds; deployment at
  $N = 2\times10^4$). Best per column in bold.}
  \label{tab:tier0-variants-op}
  \small
  \begin{tabular}{llcc}
    \toprule
    scheme & depth / adapted params & \textsc{log} & \textsc{diam} \\
    \midrule
    Fixed       & fixed / none        & $5.92 \pm 0.11$ & $47.0 \pm 0.4$ \\
    Fixed-Tune  & fixed / all         & $4.98 \pm 0.58$ & $38.3 \pm 3.4$ \\
    Fixed-Steer & fixed / adapter + readout & $3.64 \pm 0.86$ & $35.8 \pm 1.6$ \\
    Loop        & tunable / none      & $2.37 \pm 0.57$ & $14.4 \pm 7.9$ \\
    Loop-Tune   & tunable / all       & $\mathbf{1.04 \pm 0.17}$ & $\mathbf{3.26 \pm 1.17}$ \\
    LFS         & tunable / controller + readout & $1.12 \pm 0.19$ & $3.99 \pm 1.57$ \\
    \bottomrule
  \end{tabular}
\end{table}

Fig.~\ref{fig:hier-fsft} shows the ratio curves that underlie the
summary plot (Fig.~\ref{fig:hier-phase}). The effect is stronger at
$d=3$: the ratio falls to $0.98$ (\textsc{log}),
$0.77$ (\textsc{poly}), and $0.64$ (\textsc{diam}) at the largest size,
with LFS more accurate in $62$ of $75$
seed--size combinations. The one regime where Loop-Tune stays ahead at
every size is fast contraction: at $\rho = 0.5$ the ratio is
$1.11$--$1.15$ on every family, as the class definitions predict. At
$\rho = 0.7$ and $d = 2$, Loop-Tune also leads near the adaptation
sizes on \textsc{poly} and \textsc{diam} (ratios up to $1.2$) before
the curves cross.

\subsection{Damped PageRank}
\label{app:exp-anchors}

\begin{figure}[t]
  \centering
  \begin{subfigure}[t]{0.49\textwidth}
    \includegraphics[width=\linewidth]{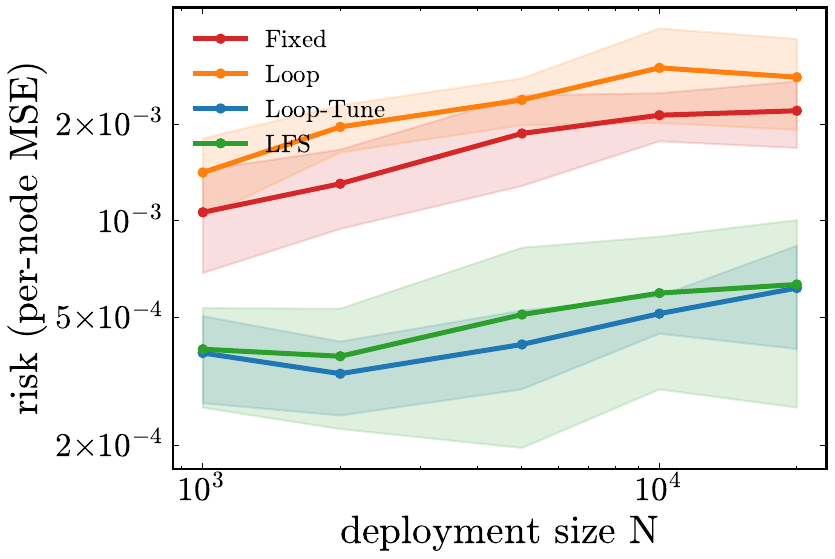}
    \caption{Risk of the four schemes across sizes.}
  \end{subfigure}\hfill
  \begin{subfigure}[t]{0.49\textwidth}
    \includegraphics[width=\linewidth]{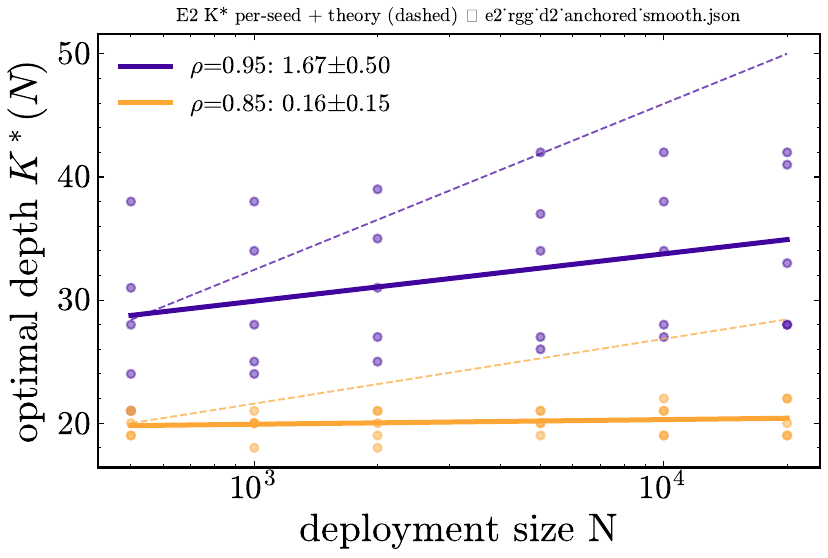}
    \caption{Risk-minimizing depth across sizes.}
    \label{fig:e2}
  \end{subfigure}
  \caption{\textbf{Damped PageRank} ($\rho = 0.95$, RGG $d{=}2$; 5
  seeds). (a) The two adapted schemes are $2.8\times$ more accurate than
  the unadapted ones at $N{=}10^3$, growing to about $4\times$ at
  $N{=}2\times10^4$. (b) With the contractive update described in the
  text, the risk-minimizing deployment depth is ordered by the
  contraction rate at every size (about $27/20/13/8$ for measured rates
  $0.92/0.80/0.63/0.42$) and grows with $\log N$ wherever the benefit of
  resolving more of the target exceeds the estimation error.}
  \label{fig:e1}
\end{figure}

The constructed tasks of Sec.~\ref{sec:exp-synthetic} prescribe a finite
number of update steps. We also repeat the study on damped PageRank, a
natural task whose target is the \emph{fixed point} of the same kind of
contractive update, so the dependency radius is set by the contraction
rate alone ($\mathcal{F}_2^{\log}$ in Definition~\ref{def:classes}).
Fig.~\ref{fig:e1} shows that the findings transfer: adaptation provides
a growing advantage over the unadapted schemes, the $\mathcal{F}_1$
control behaves like \textsc{const}, and the risk-minimizing depth grows
with the graph size at a rate ordered by the contraction rate. A plain
learned update tends to become unstable when unrolled far beyond its
training depth. For these experiments we therefore use an update of the
form $h \leftarrow a\,h_{\mathrm{enc}} + (1-a)\,R(h)$ with a learned
mixing weight $a$, which is contractive by construction and removes the
instability. On this task Loop-Tune and LFS perform alike at every size
(risk ratio $1.00$--$1.03$), because the adaptation
data are plenty and are drawn from the same distribution, so neither
condition that favors freezing (scarce data, or adaptation that disturbs
cross-size behavior) holds.

\subsection{Effect of the adaptation budget and of the size gap}
\label{app:exp-fsft}

\begin{figure}[t]
  \centering
  \includegraphics[width=0.9\textwidth]{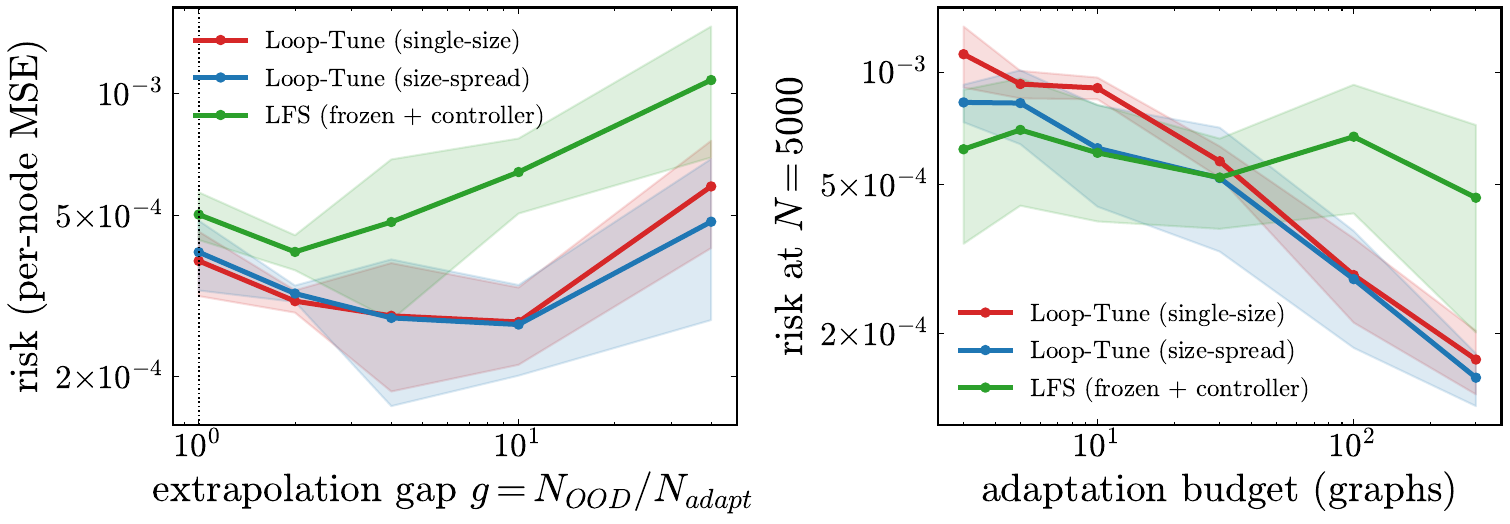}
  \caption{\textbf{Loop-Tune and LFS as the size gap
  (left) and the adaptation budget (right) vary} (damped PageRank, 3
  seeds). Left: Loop-Tune on a single adaptation size and on a spread
  of sizes perform alike at every gap up to $40\times$: on these
  homogeneous graphs, single-size fine-tuning does not damage cross-size
  behavior (cf.\ Figure~\ref{fig:epsb}). Right: the comparison crosses
  with the budget: LFS is more accurate with $30$ or fewer adaptation
  graphs (risk ratios $0.56$--$0.90$), and Loop-Tune leads once the
  budget reaches $100$ graphs.}
  \label{fig:e4e6}
\end{figure}

Proposition~\ref{prop:adapt} attributes the potential advantage of
LFS to two terms: an estimation term, which favors the
smaller adapted parameter set when labeled data are scarce, and a
transfer term, which favors the frozen processor when fine-tuning would
disturb its behavior across sizes. Fig.~\ref{fig:e4e6} isolates the
estimation term on damped PageRank, and it appears directly: with $30$
or fewer adaptation graphs LFS is more accurate (risk ratios
$0.56$--$0.90$), and Loop-Tune takes over once the budget reaches $100$
graphs, where every node provides a regression label and the data are
plentiful. On the constructed families
the crossing is reached at roughly ten times the adaptation size
(Fig.~\ref{fig:hier-fsft}). On chip placement, where the adaptation
data are also plentiful, Loop-Tune likewise leads and LFS stays within
$4\%$ of its error while training $2.9\%$ of the parameters
(Tab.~\ref{tab:chipgen}).

\subsection{Measuring how adaptation changes the processor}
\label{app:exp-a14}

\begin{figure}[t]
  \centering
  \includegraphics[width=0.95\textwidth]{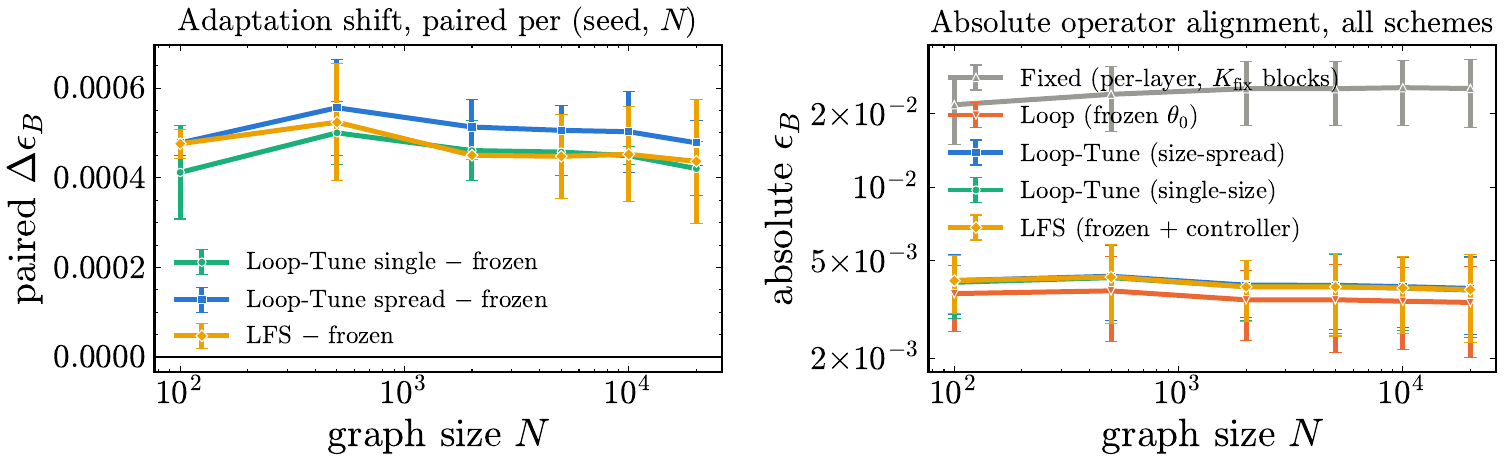}
  \caption{\textbf{How far the learned update deviates from the task
  update}, measured as
  $\lVert D(R_\theta(h)) - U_N(D(h)) \rVert^2$ along the model's own
  trajectory (damped PageRank, $\rho = 0.95$, 3 seeds). Left: change in
  this deviation caused by adaptation, relative to the frozen pretrained
  processor, matched per seed and size. Right: the deviation itself for
  all schemes. The frozen processor (Loop) is the best-aligned
  reference; fine-tuning increases the deviation uniformly across sizes;
  the layers of the fixed-depth model (Fixed), measured one by one, sit
  $5$--$10\times$ higher at every size: trained without weight sharing,
  they never learn to approximate $U_N$.}
  \label{fig:epsb}
\end{figure}

LFS can outperform Loop-Tune only if fine-tuning harms the processor's behavior on sizes beyond the adaptation range, and this is the deployment-shift term in Proposition~\ref{prop:adapt}. In this subsection we measure that harm directly. At each step of the model's trajectory, we apply the learned update and the true task update $U_N$ to the same hidden state, decode both results, and record the squared difference between them. Fig.~\ref{fig:epsb} shows this difference before and after adaptation. Fine-tuning increases it in every seed--size combination, by $9$--$22\%$. The LFS-adapted model shifts this decoded deviation by a comparable amount ($9$--$19\%$): the probe decodes through the readout, which both schemes now train, so on matched data it does not separate the two. The increase is the same at every graph size for both schemes, which explains why neither loses accuracy far beyond the adaptation sizes in Fig.~\ref{fig:e4e6}. Fine-tuning hurts size extrapolation only when the adaptation data differ structurally from the deployment data, as real design data usually do; the next experiment creates this situation deliberately, and there the two schemes separate. 

\textbf{Creating the bias deliberately.}
We adapt on random geometric graphs of average degree $4$ while
pretraining and evaluating on the standard average degree $8$. The
geometry and the task are unchanged, and only the adaptation
distribution differs (Fig.~\ref{fig:e9bias}). Adaptation on matched
data improves the processor's agreement with the task update at every
size, for Loop-Tune and LFS alike. Adaptation on the
biased data removes this improvement for Loop-Tune and pushes its
processor away from the task update near the adaptation sizes, while
LFS stays closer to the pretrained behavior at every
size. Under bias, the risk ratio of LFS to Loop-Tune
falls to $0.72$--$0.51$ (\textsc{log}) and $0.72$--$0.42$
(\textsc{diam}) across deployment sizes, with LFS more
accurate in all $5$ seeds at every size, whereas the same settings
without bias give ratios of $0.9$--$1.2$. Loop-Tune's risk increases
$2.7$--$11\times$ under bias; LFS's increases far less
($1.4$--$7\times$). The two experiments together cover both directions: on
matched adaptation data there is no harmful drift for freezing to
prevent, and on structurally biased data the drift appears and freezing
prevents it.

\begin{figure}[t]
  \centering
  \includegraphics[width=0.95\textwidth]{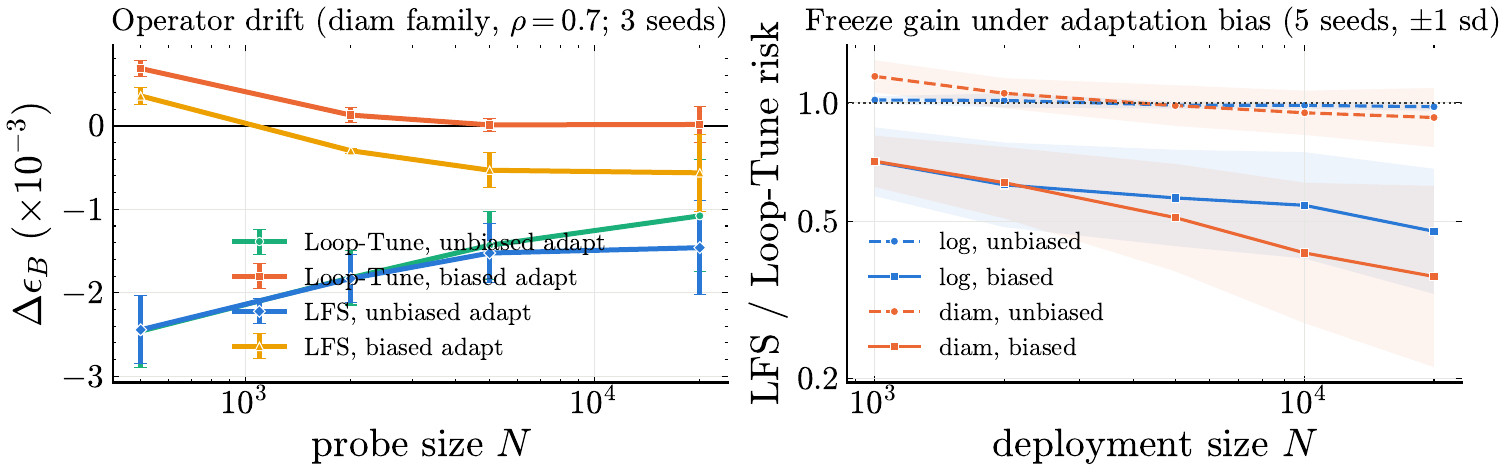}
  \caption{\textbf{Deliberately biased adaptation data} (adapt on
  degree-4 graphs at $N \approx 500$, evaluate on degree-8 graphs;
  target-supervised pretraining, $\rho = 0.7$). Left: change in the
  learned update's deviation from the task update, relative to the
  frozen processor (\textsc{diam}, 3 seeds): matched adaptation data
  improve the agreement, while biased data push Loop-Tune away from the
  task update while LFS stays closer to its pretrained behavior.
  Right: risk ratio of LFS to Loop-Tune across deployment
  sizes for \textsc{log} and \textsc{diam} (5 seeds; dashed: matched
  adaptation data, solid: biased): the bias reverses the comparison, by
  up to a factor of $2.4$.}
  \label{fig:e9bias}
\end{figure}

\subsection{Scope and failure modes}
\label{app:exp-scope}

\begin{figure}[t]
  \centering
  \begin{subfigure}[t]{0.44\textwidth}
    \includegraphics[width=\linewidth]{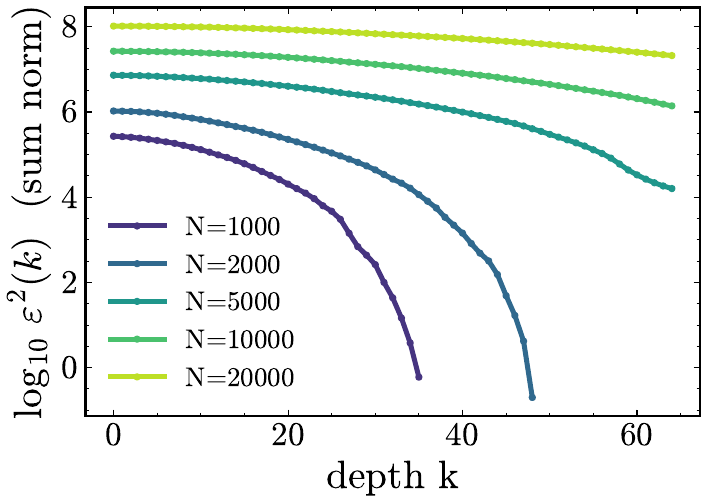}
    \caption{Error of radius-$k$ prediction (no model involved).}
  \end{subfigure}\hfill
  \begin{subfigure}[t]{0.54\textwidth}
    \includegraphics[width=\linewidth]{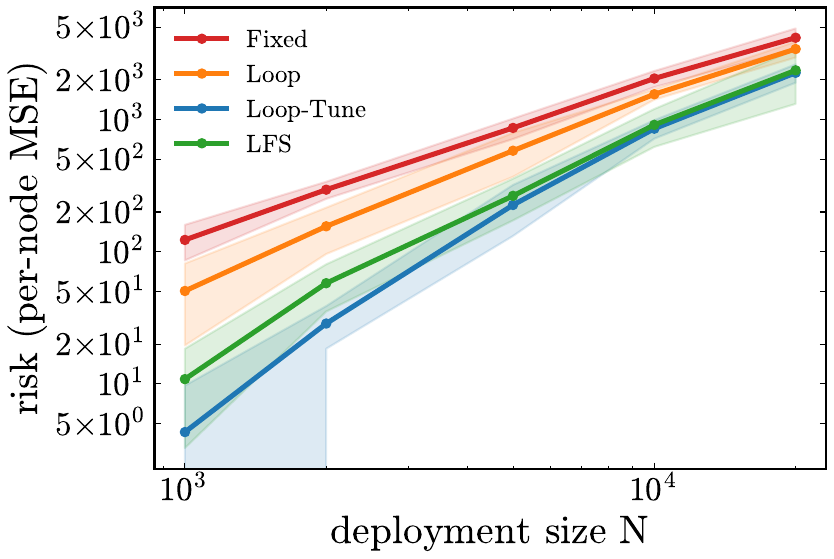}
    \caption{All four schemes on shortest paths (3 seeds).}
  \end{subfigure}
  \caption{\textbf{Outside the contractive class.} (a) For shortest-path
  distances, the error of the best radius-$k$ prediction does not decay
  geometrically; it drops only once $k$ reaches the graph diameter, of
  order $\sqrt{N}$, a property of the task itself. (b) Trained on
  this task, every scheme degrades severely with the graph size and the
  differences between schemes collapse: the failure is not specific to
  LFS, and no ordering among the schemes survives.}
  \label{fig:sssp}
\end{figure}

This subsection reports the boundaries of the framework. Shortest-path
distances are computed by a min-plus update that does not contract, and
Fig.~\ref{fig:sssp} shows the consequence: the error of
radius-limited prediction does not decay geometrically, and every scheme
fails at deployment sizes, with risk growing by one to three orders of
magnitude between $N = 10^3$ and $N = 2\times10^4$. The LFS
controller, built from feature-wise scaling, cannot express min-plus
corrections, so its advantage vanishes as well. The long-range content
of the inputs also matters. If node inputs are independent across nodes,
or if the input field on a one-dimensional graph can be inferred from
any small window, then distant information is predictable from local
information and fixed-depth models escape the error floor. Inputs whose
distant values are not locally predictable, such as the smooth field in
our construction, are therefore a requirement for the fixed-radius lower
bound (Proposition~\ref{prop:ceiling}), not a technicality. Depth extension has a stability limit. Without contraction
in the learned update, unrolling beyond about $1.5\times$ the training
depth makes risk explode; the contractive update of
App.~\ref{app:exp-anchors} removes the explosion, and
operator-aligned pretraining then extends the usable depth to about
$2.5\times$ the pretraining depth (Fig.~\ref{fig:hier-kstar}). When
the update contracts quickly ($\rho \le 0.5$), every family behaves
like $\mathcal{F}_1$: the relevant neighborhood fits inside the
pretraining radius and the differences among all four schemes vanish, as
the class definitions predict.

\end{document}